\documentclass[a4paper, twoside, 11pt]{article}

\usepackage{jair}
\usepackage{amsmath,amssymb,amsthm,color}
\usepackage{graphics,epsfig}
\usepackage{wrapfig}
\usepackage{boxit,url}
\usepackage{enumerate}
\usepackage{apacite}
\usepackage[english]{babel}
\usepackage{epsf}

\newcommand{\BlackBox}{\rule{1.5ex}{1.5ex}}  
\renewenvironment{proof}{\par\noindent{\bf Proof\ }}{\hfill\BlackBox\\[2mm]}

\newtheoremstyle{default}
{\topsep}
{\topsep}
{}
{}
{\bfseries}
{.}
{0.5em}
{\thmname{#1}\thmnumber{ #2}\textbf{\thmnote{ (#3)}}}
\theoremstyle{default}
\newtheorem{example}{Example}
\newtheorem{theorem}{Theorem}
\newtheorem{lemma}[theorem]{Lemma}

\newtheorem{remark}[theorem]{Remark}

\newtheorem{definition}[theorem]{Definition}

\newcommand{\theoremx}[2]{{\newtheorem*{theoremx{#1}}{Theorem {#1}}\begin{theoremx{#1}}#2\end{theoremx{#1}}}}
\newcommand{\lemmax}[2]{{\newtheorem*{lemmax{#1}}{Lemma {#1}}\begin{lemmax{#1}}#2\end{lemmax{#1}}}}

\newcommand{\PropA}{{\textrm{Property A}}}
\newcommand{\PropB}{{\textrm{Property B}}}
\newcommand{\PropC}{{\textrm{Property C}}}
\newcommand{\PropAcont}{{\textrm{Property A$_\textrm{cont}$}}}

\newcommand{\eqdef}{{\stackrel{\vartriangle}{=}}}

\newcommand{\rE}{{\mathbf{E}}}

\newcommand{\Real}{{\mathbb{R}}}
\newcommand{\bB}{{\mathbb{B}}}
\newcommand{\cR}{{\mathcal{R}}}
\newcommand{\cX}{{\mathcal{X}}}
\newcommand{\cRdeltastar}{{\cR_{\delta}^{*}}}

\newcommand{\cQ}{{\mathcal{Q}}}
\newcommand{\cL}{{\mathcal{L}}}

\newcommand{\cM}{{\mathcal{M}}}
\newcommand{\cH}{{\mathcal{H}}}

\newcommand{\sumn}{{\sum_{i=1}^N}}
\newcommand{\sumj}{{\sum_{i=1}^J}}
\newcommand{\sumk}{{\sum_{i=1}^K}}
\newcommand{\supp}{\mathop{\mathrm{supp}}}
\newcommand{\argmax}{\mathop{\mathrm{argmax}}}
\newcommand{\argmin}{\mathop{\mathrm{argmin}}}
\newcommand{\lebeq}{\stackrel{\lambda}{=}}
\newcommand{\symdif}{\mathop{\Delta}}
\newcommand{\aslim}{\stackrel{\mathrm{a.s.}}{\longrightarrow}}
\newcommand{\aseq}{\stackrel{\mathrm{a.s.}}{=}}
\newcommand{\klpart}[2]{{#1\log{\frac{#1}{#2}}}}

\newcommand{\pderiv}[2]{{\frac{\partial{#1}}{\partial{#2}}}}

\newcommand{\ind}{{\mathbb{I}}}
\newcommand{\cldp}{{\textrm{core}_\delta(P)}}

\definecolor{darkGreen}{rgb}{0.1,0.5,0.1}

\definecolor{orange}{rgb}{1,0.5,0.2}

\newcommand{\rnote}[1]{}
\newcommand{\mnote}[1]{}

\ShortHeadings{Foundations of Adversarial Single-Class Classification}{El-Yaniv and Nisenson}

\begin{document}
\thispagestyle{plain}



\title{On the Foundations of Adversarial Single-Class Classification}

\author{\name Ran El-Yaniv \email rani@cs.technion.ac.il \\
       \addr Department of Computer Science\\
       Technion -- Israel Institute of Technology\\
       Technion, Israel 32000
       \AND
       \name Mordechai Nisenson \email motinis@il.ibm.com \\
       \addr IBM Research -- Haifa\\
       Haifa University Campus, Mount Carmel, 31905 Haifa, Israel}


\maketitle

\begin{abstract}
Motivated by authentication, intrusion and spam detection applications
we consider single-class classification (SCC) as a two-person game between the learner and an adversary.
In this game the learner has a sample from a target distribution and the goal is to construct a
classifier capable of distinguishing observations from the target distribution from
observations emitted from an unknown other distribution.
The ideal SCC classifier must
guarantee a given tolerance for the false-positive error (false alarm rate) while minimizing
the false negative error  (intruder pass rate).
Viewing SCC as a two-person zero-sum game
we identify  both deterministic and randomized optimal classification strategies for different game variants.
We demonstrate that randomized classification can provide a significant advantage. In the deterministic setting we
show how to reduce SCC to two-class classification where in the two-class problem the other class is
a synthetically generated distribution.
We provide an efficient and practical algorithm for constructing and solving the two class problem.
The algorithm distinguishes low density regions of the target distribution and is shown to be consistent.
\end{abstract}


\section{Introduction}
In \emph{Single-Class Classification (SCC)} the learner observes a training set of sampled instances
from one \emph{target distribution}. The goal is to create a classifier that can distinguish instances emitted from
distributions other than the target distribution and unknown to the learner during training.
This SCC problem
can model many applications such as intrusion, fault and novelty detection.
For example, in an instance of an intrusion detection problem \cite<see e.g.,>{NisensonYEM03},
the goal is to create a classifier that can distinguish `legal' users from intruders based
on behaviometric or biometric patterns. This classifier can then be used to guard against illegal
attempts to gain access into protected systems or regions.

Single-class classification (also termed one-class classification) has been receiving considerable
research attention in the machine learning and pattern recognition communities.
For example, only the survey papers \cite{MarkouSingh03part1,MarkouSingh03part2,HodgeAustin04} cite, altogether, over
100 SCC papers.
Most SCC works implicitly assume that a good solution can be achieved
by identifying low density regions of the target distribution
and then, the objective is to reject sub-domains of low density.
Thus,
the main consideration in previous SCC studies has been
\emph{statistical}: how can  a prescribed false positive
rate be guaranteed given a finite sample from the target distribution.

The proposed approaches are typically
\emph{generative} or \emph{discriminative}. Generative solutions
range from full density estimation \cite{Bishop94}, to partial
density estimation such as quantile estimation \cite{LanckrietGJ02},
level set estimation \cite{BenDavidL95,SteinwartHS05} or local
density estimation \cite{BreunigKNS00}. In discriminative methods
one attempts to generate a decision boundary appropriately enclosing
the high density regions of the training set \cite{Yu05}.
In addition to such constructions, there are many empirical studies of the proposed solutions.
Nevertheless, it appears that the area
suffers from a lack of theoretical contributions and principled (empirical)
comparative studies of the proposed solutions.

Motivated mainly by intrusion detection applications, in this paper we examine the SCC
problem from an adversarial viewpoint where an adversary selects the attacking distribution.
We begin by abstracting away the statistical estimation component of the problem by considering
a setting where the learner has a very large sample from the target distribution.
This setting is modeled by assuming that the learning algorithm has precise knowledge of the target distribution.
While this assumption would render almost the entire body of SCC literature superfluous,
it turns out that
a significant and non-trivial \emph{decision-theoretic} component of the adversarial SCC problem remains --
one that has
so far been overlooked.
For a discrete version of the SCC problem we provide an in depth analysis of adversarial SCC and identify optimal strategies
for variants of the problem depending on whether or not the learner can play a randomized strategy
and on various constraints on the adversary. As a consequence of this analysis, it can be demonstrated
that a randomized learner strategy can be superior on average to standard deterministic classification.
For an infinitely continuous version of this game we provide a simple and consistent SCC algorithm
that implements the standard
low-density rejection by reducing the SCC problem to two-class soft classification.

The body of this
paper contains the principal results that are simpler to present. The appendices contain some of
the more technical proofs.
to the presented results.
An earlier version of this work containing a subset of the results was presented at NIPS  \cite{ElYanivN06}.
Extensions to this work can be found in the thesis of \cite{Nisenson2010}.

\section{Problem Formulation}
\label{sec:formulation}
We define the adversarial \emph{single-class classification (SCC)} problem as a two-person
zero-sum game between the
\emph{learner} and an \emph{adversary}. The learner receives a training sample of examples from
a \emph{target distribution} $P$ defined over some space $\Omega$.
On the basis of this training sample, the learner should select a rejection function
$r: \Omega \to [0,1]$, where for each $\omega \in \Omega$, $r(\omega)$ is the probability with
which the learner will reject $\omega$.
On the basis of any knowledge  of $P$ and/or $r(\cdot)$, the adversary
selects an \emph{attacking distribution} $Q$, defined over $\Omega$.
Then,
a new example is drawn from $\gamma P + (1-\gamma)Q$, where
$0<\gamma<1$, is a \emph{switching probability} unknown to the learner.

The \emph{rejection rate} of the learner, using a rejection function $r$, with respect to any distribution
$D$ (over $\Omega$), is $\rho(r,D) \eqdef \rE_D \{r(\omega)\}$.
The two main quantities of interest here are the \emph{false positive rate} (type I error) $\rho(r,P)$, and
the \emph{false negative rate} (type II error) $1-\rho(r,Q)$.
Before the start of the game, the learner receives a tolerance
parameter $0<\delta<1$, giving the maximally allowed false positive
rate. A rejection function $r(\cdot)$ is \emph{valid} if its false
positive rate satisfies the constraint $\rho(r,P) \leq \delta$. A valid rejection function (strategy) is
\emph{optimal} if it guarantees the smallest false
negative rate amongst all valid strategies.

This setting conveniently models various SCC applications and in particular, intrusion detection problems.
For example, considering biometric authentication, the false alarm rate $\rho(r,P)$ is the rejection
(failed authentication) rate of the legal
users and $\rho(r,Q)$ is the rejection rate of intruders, which should be maximized.

\begin{remark}
Clearly, a dual SCC problem can be formulated
where a sufficiently high intruder rejection rate must be guaranteed and the false alarm rate should be minimized.
We briefly discuss this dual problem and its relation to the ``primal'' in Section~\ref{sec:dual}.
Other types of SCC problems can be considered where the loss is a function of the type I and type II errors. For example, one may be interested in
minimizing a convex combination of these errors.
Any such loss function can be handled using our definition and searching for the $\delta$ for which the SCC solution
optimizes the desired loss function.
\end{remark}


Our analysis begins by focusing on the Bayes decision theoretic version of the SCC problem
in which the learner knows
the target distribution $P$ precisely.
The problem is thus viewed as a two-person zero sum game where the payoff to the learner is $\rho(r,Q)$.
The set $\cR_{\delta}(P) \eqdef \{r : \rho(r,P) \leq \delta \}$ of valid rejection functions
is the learner's strategy space.
We denote by $\cQ$ be the strategy space of the adversary, consisting of all
allowable distributions $Q$ that can be selected by the adversary.\footnote{The game can be
expressed in  `extensive form' (i.e., a game tree) where
in the first move the learner selects a rejection function, followed by a chance move to determine
the source (either $P$ or $Q$) of the test example (with probability $\gamma$). In the case where $Q$ is selected,
the adversary chooses (randomly using $Q$) the test example. In this game the choice of $Q$ depends on
knowledge of $P$ and $r(\cdot)$.}


We are concerned with optimal learner strategies
for game variants distinguished by the adversary's knowledge of the
learner's strategy, $P$ and/or of $\delta$ and by other limitations on $\cQ$.
We also distinguish a special type of this game, which we call the \emph{hard setting} in which the learner
is constrained to employ only deterministic reject functions; that is,
$r: \Omega \to \{0,1\}$, and such rejection functions are termed ``hard.''
The more general game defined above (with ``soft'' functions) is called the
\emph{soft setting}.
As far as we know,
only the hard setting has been considered in the SCC literature thus far.
The reason for considering soft rejection functions is that they can achieve significant advantage in terms
of type II error reduction. Later on in Section~\ref{sec:numerical} we numerically demonstrate such error reductions.

For any rejection function, the learner can reduce the type II error by rejecting more (i.e., by increasing
$r(\cdot)$).
Therefore, in the soft setting for an optimal $r(\cdot)$ we must have $\rho(r,P) = \delta$ (rather than
$\rho(r,P) \leq \delta$). It follows that the switching parameter $\gamma$ is immaterial to the selection of
an optimal strategy.


Given an adversary strategy space, $\cQ$, we define the set
$\cRdeltastar(P)$  of optimal valid rejection functions as
$\cRdeltastar \eqdef \{r \in \cR_{\delta}(P) \ : \min_{Q \in \cQ} \rho (r,Q) =
\max_{r' \in \cR_{\delta}(P)} \min_{Q' \in \cQ} \rho(r',Q') \}$.\footnote{For certain strategy spaces, $\cQ$, it may be necessary to consider the infimum rather than
the minimum. In such cases it may be necessary to replace `$Q \in \cQ$'
(in definitions, theorems, etc.)
with `$Q \in cl(\cQ)$', where $cl(\cQ)$ is the closure of $\cQ$.}
We note that $\cRdeltastar$ is never empty in the cases we consider.

\section{Related Work}
One-Class Classification is often given different names, depending on the desired use. For example,
other common names include outlier detection, fault detection and novelty detection.
Historically, one of the
earliest works is due to~\citeA{Grubbs69} who considered in-sample outlier detection. Grubbs calculates
a cut-off statistic for determining outliers in the 1-dimensional Gaussian
case at the 5\%, 2.5\% and 1\% significance levels within samples of various sizes.
\citeA{Minter75} appears to be the first to use the term ``single-class classification''.
Minter starts from a fairly standard two-class approach, assuming that there is a class of interest (class 1)
and a class of ``others'' (class $\emptyset$). Given the switching parameter $\gamma$ (which is the a priori probability
of class 1), Minter gives the rule to accept a point $x$ iff $\gamma\Pr\{x|1\} \geq (1-\gamma)\Pr\{x|\emptyset\}$,
which is equivalent to $\gamma\Pr\{x|1\} \geq \frac{1}{2}\Pr\{x\}$. It is assumed
that both $\gamma$ and $\Pr\{x\}$ are known or can be estimated from historical data, leaving
the problem of estimating $\Pr\{x|1\}$ from the given sample.
While, technically, only a sample from the class of interest is given,
the additional assumptions make this a modified form of a two-class
problem.\footnote{This differs from more recent works, where $\gamma$ and $\Pr\{x\}$ are assumed to be unknown
(whereby the learner's knowledge is much more restricted), and the type I error is required to not exceed a bound, $\delta$, which
is the setting we use in this work.}
These are the
earliest explicit works we have found.
Note that statisticians have long been considering the \emph{two-sample problem}, which is similar but
perhaps simpler. One can view the SCC problem as an extremely unbalanced instance of the two-sample problem
that prevents using the standard statistical hypothesis testing techniques.

Since virtually all prior works on SCC that we have encountered deal
with how to approximate a low-density rejection strategy given a set, $\{x_1,\ldots,x_n\}$, of training points,
sampled from the class of interest, we will focus our review here on such methods.

We begin with discussing support, quantile and level-set estimation. Support estimation aims to estimate the support
of a density $p$. In terms of outlier detection, the goal is clear: a point falling outside the estimated support is
taken to be an outlier. One of the simpler methods, analyzed by \citeA{DevroyeWise80}, is to
estimate the support as $\hat{S}_n = \bigcup_{i=1}^n B(x_i, \epsilon_n)$, where
$B(x,a)$ is a closed ball centered at $x$ with radius $a$ (i.e. $||x'-x|| \leq a$, for some norm $||\cdot||$),
and $\epsilon_n$ is a (vanishing) sequence of smoothing parameters.
In quantile estimation,
the goal is to find a set $U(\beta)$ such that $\lambda(U(\beta)) = \inf_S\{\lambda(S): P(S) > \beta\}$,
where $\lambda$ is a real valued function. For our purposes, we take $\lambda$ as the Lebesgue measure,
in which case the problem is also called \emph{minimum volume estimation}.
When $\beta = 0$ this becomes support estimation, and when $\beta = 1-\delta$ this problem is the same as low-density rejection.
In level-set estimation,
the goal is to approximate the set $\cL(t) = \{x: p(x) > t\}$ (or alternatively as $\{x: p(x) \geq t\}$). Of course,
level-set estimation can be used for support estimation by taking $t=0$ or
by taking $t = t_n$ as a sequence which approaches zero \cite<see>{Cuevas97}.
Clearly, level-set estimation approximates the low-density rejection strategy when $P(\cL(t)) = 1 - \delta$.
A significant amount of prior SCC works have focused on minimum volume and level-set estimation. We distinguish between
explicit and implicit methods, where explicit methods try to directly solve one of the problems, and implicit methods which use a heuristic which may or may
not give the desired result. We note that whether the method is explicit or implicit
is not necessarily an indicator of whether the underlying model is generative or discriminative,
although there is a clear tendency for explicit methods to be generative. Transformations from the one-class setting to the two-class setting
tend to be implicit and discriminative. We will consider minimum volume estimation approaches first and
then look at various level-set estimation results. Finally we will examine other results, including transformations to the two-class setting.

Minimum volume estimation has been a favored approach at solving the SCC problem in the literature.
This perhaps is due to two works which reused the popular Support Vector Machine \cite<SVM, see>{Vap98} from two-class classification problems.
The earlier work \cite{TaxD99} sought to fit the sample data inside a sphere of minimal radius, a solution
they called the Support-Vector Data Description (SVDD). Specifically,
given a sphere with center $a$ and radius $R$, the error function to be minimized is $R^2 + C\sum_i \xi_i$,
under the constraints $(x_i - a)^T(x_i - a) \leq R^2 + \xi_i$, where $C$ is a regularization term which
relates to the type I error. Outliers in the sample data would lie on, or outside the sphere (and have $\xi_i > 0$).
The kernel trick was then employed to allow for solving the problem in a higher dimensional feature space. They note
that polynomial kernels do not result in small volumes in the input space, as points distant from the origin tend
to have high error values. They found that Gaussian kernels worked well. The type I error can be estimated
from the number of support vectors divided by the sample size, $n$, where the support vectors are the points lying on the sphere (i.e. they define the sphere's boundary).
Changing the regularization parameter $C$, or the bandwidth parameter of the Gaussian kernel, can be used to control the trade-off between the volume
of the sphere and the number of support vectors. In a follow up work, \citeA{TaxD01}
show how samples from a uniform distribution can be used to optimize for both parameters simultaneously. The second
work \cite{Scholkopf01estimatingthe} introduced what is commonly called the One-Class Support Vector Machine (OC-SVM).
The technique used is that of a standard two-class SVM where the second class is the origin (in \emph{feature space}).
In other words, a hyper-plane is sought which maximizes the soft-margin between the origin and the sample points. Points
lying on the ``wrong'' side of the hyper-plane are outliers. The kernel trick can also be employed for OC-SVM. Sch\"{o}lkopf et.~al show that for kernels $k(x,y)$ that depend only on $x-y$, such as the Gaussian kernel,
the solutions found by OC-SVM and SVDD are identical.
They further showed that the value $\nu = \frac{1}{nC}$, where $C$ is the regularization parameter in the SVM equation, is
an upper bound on the number of outliers, a lower bound on the number of support vectors, and that for probability measures $P$
without discrete components, asymptotically the number of outliers and support vectors are equal, in probability.
\citeA{Vert06} correctly point out that while OC-SVM can guarantee the type I error, no guarantees are made regarding consistency of the result (i.e.,
whether the result converges to a region of minimum volume). This same point is
valid for SVDD as well. Indeed, the poor performance of SVDD using polynomial kernels is sufficient proof that the minimum volume set (in the original feature space) is not found.
Thus, both of these approaches are implicit, as they do not explicitly solve for the minimum volume set.
Similar results for the Minimax Probability Machine
(where the type I error is bounded but the resulting set does not necessarily have the minimum volume) are provided by Lanckriet et. al \cite{Lanckriet02arobust,LanckrietGJ02}.
\citeA{ScottNowak06} overcome these limitations where they use
Empirical Risk Minimization to prove consistency (in a distribution free manner) and convergence rates
of $\left(\frac{\log n}{n}\right)^{\frac{1}{d}}$
using Structural Risk Minimization for trees
 (these results aren't distribution free; specifically there is a requirement which can be satisfied if $p$ has no plateaus).
\citeA{Scott07} expands on this analysis, which served as the basis for the 2-class SVM approach used in \cite{Davenport06learningminimum}, where the second class
is the uniform distribution. The results significantly outperformed those of OC-SVM (i.e. a significantly smaller volume was found for approximately
the same type I error).

We now turn our attention to level-set estimation.
Let $\cL_n(t)$ be the estimation of $\cL(t)$ given the $n$ sample points. One of the most common error measures is $\lambda(\cL(t) \symdif \cL_n(t))$,
where $\lambda$ is the Lebesgue measure and $\symdif$ is the symmetric difference (i.e. $A \symdif B = (A \setminus B) \bigcup (B \setminus A)$).
Another common measure is $H_P(\cL(t),t) - H_P(\cL_n(t),t)$, where $H_P(S,t) = P(S) - t\lambda(S)$ is the excess mass of $S$.
Both of these measures are non-negative and equal to zero at the optimal solution.
Much of the prior work which explicitly solves
the level-set estimation problem shows consistency by proving that as $n$ goes to infinity, one of these two measures goes to zero. Most
recent work focuses on calculating convergence rates under various conditions on the density $p$. One of the most common techniques for level-set
estimation is the \emph{plug-in estimate} where $\cL_n(t) = \{x: \hat{p}_n(x) > t\}$,
for a density estimate $\hat{p}_n$ of $p$. The kernel density estimate \cite{Parzen:1962} is most often used. For a thorough analysis of the plug-in estimate
(in terms of consistency and convergence rates) see \citeA{Cuevas97,Cadre97,Rigollet08}. Interestingly,
the SCC community appears to have been inclined to pursue alternate and novel approaches over the straight-forward use
of the kernel density estimate as part of the plug-in estimator. It must be stressed that these approaches
have largely been implicit, in the sense that they are based on either a heuristic or some other approximation, and consistency is not proven. For example,
\citeA{BreunigKNS00} develop a measure they call the Local Outlier Factor (LOF). LOF is calculated based on a smoothed k-nearest-neighbor
distance, where the LOF is calculated as an average ratio of these distances between the neighbors of a point and the point itself. In other words,
the LOF is calculated so that objects ``deep within a cluster'' will have a LOF of approximately 1, while objects near edges of clusters or far from other points will have large values.
This seems to be a heuristic way of estimating $f(p)$ where $f$ is hoped to be a monotonically decreasing function. \citeA{Hempstalk2008} use the plug-in estimate
approach where they use a rather different way of establishing $\hat{p}_n$. Using Minter's notation from above, they generate an artificial distribution for class $\emptyset$,
and then it follows from Bayes Theorem that:
$$
\Pr\{x|1\} = \frac{\Pr\{\emptyset\}\Pr\{1|x\}}{\Pr\{1\}\Pr\{\emptyset|x\}}\Pr\{x|\emptyset\}.
$$
Since the artificial distribution is known, and the prior can be controlled, $\Pr\{x|1\}$ can be estimated from $\Pr\{1|x\}$,
which is estimated using class-probability estimation techniques, specifically bagged trees with Laplacian smoothing. In practice, they use
a density estimate of $p$ to establish the density for the artificial set. While the technique is
certainly interesting, it would be of great interest to see if consistency or convergence rates could be proven. \citeA{Vert06} demonstrated
that one need not estimate the density directly in order to determine the level-set. They prove
that an SVM, with a convex loss function and Gaussian kernel with a ``well-calibrated bandwidth $\sigma$,'' can produce
an estimate $\cL_n(t)$, such that $\lim_{n \to \infty} H_P(\cL(t),t) - H_P(\cL_n(t),t) = 0$, in probability.
\citeA{SteinwartHS04} provide convergence rates when using L1-SVM for the error measure $\mu(\cL(t) \symdif \cL_n(t))$,
where $\mu$ is a reference probability distribution.

Finally, we consider other works, starting with transformations to the two-class setting. All of these approaches
rely on the creation of a second class in the vicinity of the target class. Examples of this are \cite{CounterExample07} where
SVM is used to separate between the two classes, and \cite{GeneticOC09}, where genetic programming is used and the fitness
function accounts for overlap between the two classes. Other works, such as \cite{RatschMSM02}, look at how
boosting can be applied in the one-class setting. A recent and interesting work is by \citeA{Juszczak_09_spanning_tree}, which uses the premise
that the target class should largely be continuous; in other words, if two points belong to the target class, there should be a path from one to the other.
For points which are very close to each other, we may expect this to be a straight line. They propose building a minimum spanning tree covering the data,
and test membership to the target class by testing the distance of a point to the tree. Since the continuity assumption may be violated for points in different
clusters, they allow for the removal of edges in the tree, where longer edges are better candidates for removal. They also allow for a form of dimensionality reduction by
removing the shortest paths in the tree. The approach has very good performance on the tested data sets, and it would be of great interest to
see if the authors can develop consistency or other theoretical results for it.


\

\section{An Informal Look - an Investment/ROI Analogy}

To gain some insight into the one-class classification setting, we now
describe an  analogous investment game.
The learner is given an amount of money to invest, $\delta$.
There are $N$ assets which can be invested in, with a cost of $p_i$ to invest in asset $i$.
For each asset $i$, the learner
purchases an amount $r(i) \in [0,1]$ (i.e., from none to all of an asset)
and then sells it at a price $q_i$, determined by the adversary.  Any monies
not invested are lost.
Since the initial wealth is $\delta$, the allocation strategy $r(\cdot)$ must satisfy
$\sum_i r(i) p_i \leq \delta$. The overall return to be maximized is
$\sum_i r(i) q_i$.

Clearly, the Return-On-Investment (ROI) for asset $i$ is $\frac{q_i}{p_i}$,
and thus the learner should invest in assets which have the highest ROI (where free assets are taken
to have infinite ROI).
In the SCC setting, the fact that
the learner must select the investment strategy, $r(\cdot)$, before the adversary determines the selling prices, clearly makes
this a difficult proposition. Had we reversed the order, and the adversary were to determine the
selling prices first, we would have a two-class classification problem (i.e., the learner,
with full knowledge of both classes,
is to minimize
type II error subject to a maximum type I error). In this case, the learner's optimal investment strategy would be clear:
\begin{quote}
The learner shouldn't invest in an asset $k$, unless all assets with a higher ROI than $k$ have already been purchased.
\end{quote}
Note that while
this strategy
applies to the soft setting ($r(i) \in [0,1]$), the optimal solution is very nearly identical to that of the hard solution ($r(i) \in \{0,1\}$), with the only difference being
that any left over money is invested. How does this investment strategy translate from the two-class classification setting to our
original one-class classification setting, where the learner must invest without knowing the ROI values? Clearly,
if the adversary's strategy space has some inherent constraints on the relative ROI of assets, then the learner could take advantage of them.
For example, in the simplest case, if the adversary's strategy space
enforces an ordering on the ROI values, for example $j < k \Rightarrow \frac{q_j}{p_j} < \frac{q_k}{p_k}$, then the learner can invest optimally without knowing $Q$. However,
the less the adversary's strategy space constrains the relative ROI of assets,
the more difficult the learner's task is. We would intuitively expect that, in the face of an adversary determined to minimize
the learner's return, that less constraints on the adversary would force the learner to diversify his investment.
In the
extreme case of no constraints at all on the adversary, the learner should purchase the same amount of every non-free asset.\footnote{Note
that this is different than `dollar-cost averaging'; the same amount of money isn't spent on each asset, rather the same absolute amount of each asset is purchased.
This guarantees the learner a total ROI of at least $1$ (i.e., for every dollar invested, a dollar is earned upon selling).} We also
note that the more the learner diversifies, the ``further'' his investment strategy becomes relative to the optimal
two-class strategy (in accordance to \emph{known} ROI values).

\section{On the Optimality of Monotone and Low-Density Rejection Functions}
\label{sec:characterization}
The vast majority of the literature on SCC deals with various techniques for implementing
the Low-Density Rejection Strategy (LDRS). This raises the question of whether
such a strategy is optimal or not, and under what conditions may it be reasonable to use such
a strategy.
Since we are interested in adversarial applications, \emph{worst-case} performance
is a natural measure for us to consider. For example, if one considers an authentication system
every attempt to gain access results in either access being granted or an alarm being fired.
From a worst-case perspective, we should expect a sophisticated intruder to be
capable of spying upon legitimate use of the system for some period of time and
seeing what events or patterns should provide access. Thus, it is more likely that the intruder will attempt
to  enter a highly probable event
in order to gain access, rather than a low-probability event.
In fact, the intruder's distribution
could be even more concentrated on the highly-probable events than the user's!

Viewed in this perspective, it is not at all clear at the outset that the standard LDRS approach
to SCC is the best for adversarial applications.
By constraining the adversary's strategy space to one
where all of the distributions are tightly concentrated on the highly-probably events under $P$,
low-density-rejection may not be an optimal strategy for the learner. In the extreme case
where the adversary always plays the most probable event under $P$, the adversary would always be able
to gain access if the learner plays the low-density-rejection strategy, while potentially the learner
could completely deny the adversary access if $\delta$ is greater than the probability for that event. Clearly,
the nature of the constraints placed on the adversary is critical not only in terms of whether
LDRS is optimal, but also in terms of the error that is achievable (both by LDRS and by other strategies).
Here we address the former issue, which we feel is of particular relevance considering
the large body of existing work which examines approximating low-density
rejection functions\footnote{See, e.g., \cite{Scholkopf01estimatingthe,Cuevas97,Cadre97,BreunigKNS00}.}
that can be leveraged in solving practical problems, and leave the latter for future research.

The partially good news is that low-density rejection is worst-case optimal if the learner is
confined to ``hard'' decisions and when the adversary is strong enough in the sense that
her strategy space is sufficiently large as shown in Theorem~\ref{thm:LDRS_optimal}.
However, as we demonstrate in Section~\ref{sec:games}, LDRS is inferior in general
to the optimal soft strategy. Thus, by playing a randomized strategy,
a very significant gain can be achieved.

In this section, we assume a finite support of size $N$; that is, $\Omega = \{1,\ldots,N\}$
and $P \eqdef \{p_1,\ldots,p_N\}$ and $Q \eqdef \{q_1,\ldots,q_N\}$ are probability mass functions.
Note that this assumption still leaves us with an infinite game because the learner's pure strategy space,
$\cR_{\delta}(P)$, is infinite. Extensions to infinite support ($N \to \infty$) for many of the finite support results are given in \citeA{Nisenson2010}.
A simple observation is that for any $r \in \cRdeltastar$ there exists $r' \in \cRdeltastar$ such that
$r'(i) = r(i)$ for all $i$ such that $p_i > 0$ and for zero probabilities, $p_j = 0$,
$r'(j) = 1$. We thus assume w.l.o.g. that $p_i > 0$ for all $i \in \Omega$.

While the low-density rejection strategy implies an assumption that
lower probability events should be completely rejected, we instead examine a weaker, but perhaps
more useful, condition.
Intuitively, it seems plausible that the learner should not assign higher rejection values
to higher probability events under $P$. That is, one may expect that a reasonable rejection function
$r(\cdot)$ would be monotonically decreasing with probability values.
In the ROI analogy, we would state this as ``the learner
should prefer cheaper assets to more expensive ones.'' This is appealing, as more of a cheaper
asset can be purchased for the same amount of money than a more expensive asset, and a lower selling price is necessary to
achieve the same ROI. We now define two types of monotonicity.

\begin{definition}[Monotonicity]
A rejection function $r(\cdot)$ is \emph{monotone} if $p_j < p_k \Rightarrow r(j) \geq r(k)$.
A monotone rejection function $r(\cdot)$ is \emph{strictly monotone} if
$p_j = p_k \Rightarrow r(j) = r(k)$.
\end{definition}
We note that completely rejecting null-events under $P$ (i.e., $p_j = 0 \Rightarrow r(j)=1$)
does not break strict-monotonicity so our assumption that there are no null events under $P$ is taken w.l.o.g. Surprisingly, optimal monotone strategies are not always guaranteed
as shown in the following example.
%

\begin{example}[Non-Monotone Optimality]
In the hard setting, take $N=3$, $P = (0.06, 0.09, 0.85)$ and $\delta = 0.1$.
The two $\delta$-valid hard rejection functions are
$r' = (1,0,0)$ and $r'' = (0,1,0)$.
Let $\cQ = \{Q = (0.01, 0.02, 0.97) \}$.
Clearly $\rho(r',Q) = 0.01$ and $\rho(r'',Q) = 0.02$ and therefore,
$r''(\cdot)$ is optimal despite breaking monotonicity.
More generally, this example holds if
$\cQ = \{ Q : q_2 -q_1 \geq \varepsilon\}$ for any $0 < \varepsilon \leq 1$.

In the soft setting, let
$N = 2$, $P = (0.2, 0.8)$, and $\delta = 0.1$. We note that
$\cR_{\delta}(P) = \{ r^\varepsilon = (0.1 + 4\varepsilon, 0.1 - \varepsilon) \}$, for $\varepsilon \in [-0.025, 0.1]$.
We take $\cQ = \{Q = (0.1, 0.9)\}$. Then $\rho^\epsilon(Q) = 0.1 + 0.4\varepsilon - 0.9\varepsilon = 0.1 - 0.5\varepsilon$.
This is clearly maximized when we minimize $\varepsilon$ by taking $\varepsilon = -0.025$, and
then the optimal rejection function is $(0, 0.125)$, which clearly breaks monotonicity.
This example also holds for $\cQ = \{ Q: q_2 \geq c q_1 \}$ for any $c > 4$.
\end{example}

This example naturally raises the question of which conditions are necessary or sufficient for optimal
monotone strategies to be guaranteed. To motivate our sufficient condition for optimality
($\PropA$ below), recall the intrusion detection setting discussed in the beginning of this section.
There the adversary is constrained to distributions that are tightly concentrated on the highly probably events
under $P$. In this case, since low probability events are scarcely ``attacked'' by the adversary, the
optimal learner would not waste rejection ``resources'' on low probability events.
In other words, in such cases monotone rejection functions aren't optimal.
This begets the question if monotone rejection functions are optimal when
the adversary is not constrained from attacking low probability events.


\begin{definition}[$\PropA$]
Let $P$ be a distribution and $\cQ$ be a set of distributions. If for all $p_j < p_k$
and $Q \in \cQ$ for which $q_j < q_k$, there exists a distribution $Q' \in \cQ$
such that for all $i \neq j,k$, $q'_i = q_i$ and
$q_j + q'_j \geq q_k + q'_k$, then
$\cQ$ possesses $\PropA$ w.r.t.~$P$.

\end{definition}

\begin{example}[Possession of $\PropA$]
Let $P$ be any distribution over $\Omega$.
Let $\cQ_1 = \{U\}$, where $U$ is the uniform distribution
over $\Omega$. Then $\cQ_1$ has $\PropA$ w.r.t.~$P$ since $q_j < q_k$ is never true.
Similarly, let $\cQ_2$ be the set of all distributions (if $Q$ is a distribution over $\Omega$, then $Q \in \cQ_2$).
Then $\cQ_2$ also has $\PropA$ w.r.t.~$P$.
If $P \neq U$, and $\cQ_3 = \{P\}$, then, $\cQ_3$ doesn't possess $\PropA$ w.r.t $P$.
\end{example}


The following theorem ensures that there exists an optimal  monotone rejection function
whenever $\cQ$ satisfies $\PropA$. In such cases
the learner's search space can be conveniently confined to monotone strategies.

\begin{theorem}[Optimal Monotone Hard Strategies]
\label{thm:hard-decision}
When the learner is restricted to hard-decisions
and $\cQ$ satisfies $\PropA$ w.r.t.~$P$, then there exists a monotone $r \in \cRdeltastar$.
\end{theorem}

Theorem~\ref{thm:hard-decision} only concerns the hard setting where $r$ is a zero-one rule.
The following  $\PropB$ and the accompanying Theorem~\ref{thm:chr-monotonic-equality}
treat the more general soft setting.

\begin{definition}[$\PropB$]
Let $P$ be a distribution and $\cQ$ be a set of distributions. If for all $0 < p_j \leq p_k$
and $Q \in \cQ$ for which $\frac{q_j}{p_j} < \frac{q_k}{p_k}$, there exists $Q' \in \cQ$ such that
for all $i \neq j,k$, $q'_i = q_i$ and $\frac{q'_j}{p_j} \geq \frac{q'_k}{p_k}$, then
$\cQ$ possesses $\PropB$ w.r.t.~$P$.

\end{definition}

\begin{example}[Possession of $\PropB$]
Let $P$ be any distribution over $\Omega$.
Let $\cQ_1 = \{U\}$, $\cQ_2$ be the set of all distributions and
$\cQ_3 = \{P\}$. All three sets, $\cQ_1$, $\cQ_2$ and $\cQ_3$, have $\PropB$ w.r.t.~$P$.
\end{example}

Recalling our informal investment analogy,
if the strategy
space of the adversary satisfies $\PropB$, then cheaper assets always have the potential for higher ROI (and
equally priced assets have equal ROI opportunities). If this is the case, then
Theorem~\ref{thm:chr-monotonic-equality}
states that
there is an optimal investment strategy (that maximizes the overall return), which never purchases
more of an expensive asset than a cheaper one and always invests identically
in equally priced assets.

\begin{theorem}[Optimal Monotone Soft Strategies]
\label{thm:chr-monotonic-equality}
~\\
If $\cQ$ satisfies $\PropB$ w.r.t.~$P$, then there exists an optimal strictly monotone rejection function.
\end{theorem}
\begin{remark}
It is not hard to prove that a slightly stronger version of
$\PropA$ implies $\PropB$. The stronger version of $\PropA$ is
that the property also holds
when $p_j=p_k$ (rather than only for $p_j < p_k$).
\end{remark}



In the remainder of this section we only consider the hard setting.
Theorem~\ref{thm:hard-decision} tells us that there exists an optimal rejection function
in the set of monotone rejection functions provided that $\PropA$ holds.
Obviously, to be optimal the rejection function should reject as much as possible
up to the $\delta$ bound.
We now show that if $\cQ$ is sufficiently rich (satisfying $\PropC$ below) then
any ``low-density rejection function'' is optimal.

\begin{definition}[Low-Density Rejection Function (LDRF) and Strategy (LDRS)]
A hard, $\delta$-valid, monotone rejection function $r(\cdot)$  is called a \emph{low-density rejection function}  if its
$\rho(r,P)$ is maximal among all hard, monotone $\delta$-valid rejection functions.
The strategy of selecting any LDRF is called the \emph{low-density rejection strategy (LDRS)}.
\end{definition}

%

\begin{definition}[$\PropC$]
Let $P$ be a distribution. We say that the set $\cQ$ satisfies $\PropC$ (w.r.t.~$P$)
if for each $p_j = p_k$ and $Q \in \cQ$, there exists $Q' \in \cQ$ such that
$q'_j = q_k$ and $q'_k = q_j$ , and for all other events, $Q'$ identifies with $Q$.
\end{definition}

Some intuition about $\PropC$ can be gained by considering some adversary strategy space $\cQ$.
First note that by expanding $\cQ$ to satisfy $\PropC$ the adversary can only be strengthened.
The property ensures that the adversary can take advantage of situations where the learner
doesn't identically treat equally probable events under $P$.
When the adversary is sufficiently strong in this sense  we are able to show that LDRS
dominates any monotone rejection function. Therefore, if $\cQ$ also satisfies $\PropA$, in which case
there exists an optimal monotone rejection function (Theorem~\ref{thm:hard-decision}), then LDRS is optimal.
This is summarized in the following theorem.

\begin{theorem}[LDRS Optimality]
\label{thm:LDRS_optimal}
Let $r^*$ be an LDRF. Let $r$ be any monotone $\delta$-valid
rejection function. Then, $r^*$ dominates $r$,
\begin{equation}
\label{eq:domination}
\min_{Q \in \cQ} \rho(r^*,Q) \geq \min_{Q \in \cQ} \rho(r,Q),
\end{equation}
for any $\cQ$
satisfying $\PropC$. Thus, if $\cQ$ possess both $\PropA$ and $\PropC$ w.r.t.~$P$, then LDRS is
hard-optimal.
\end{theorem}

\begin{example}[Violating $\PropC$ Breaks Domination]
We illustrate here a violation of $\PropC$ may result in a violation of
the domination inequality (\ref{eq:domination}) in
Theorem~\ref{thm:LDRS_optimal}.
Let $N = 5$, $P = (0.02, 0.03, 0.05, 0.05, 0.85)$,
and $\delta = 0.1$. Then the two $\delta$-valid LDRS rejection functions are $r = (1,1,1,0,0)$
and $r' = (1,1,0,1,0)$. Let $\cQ = \{ Q : q_3 - q_4 > \varepsilon \}$ for some $0 < \varepsilon < 1$.
Clearly, $\cQ$ does not satisfy $\PropC$.
For any $Q \in \cQ$, $\rho(r,Q) - \rho(r',Q) = q_3 - q_4 > \varepsilon$, and therefore,
$\min_{Q \in \cQ} \rho(r',Q) < \min_{Q \in \cQ} \rho(r,Q)$. Thus, the monotone function
$r$ dominates the LDRF, $r'$. Hence, LDRS isn't optimal because $r'$ could be chosen.
\end{example}

\section{The Omniscient Adversary: Games, Strategies and Bounds}
\label{sec:games}

We next turn our attention to the power of the adversary, an issue that hasn't been emphasized
in the SCC literature, but has crucial impact on the relevancy of SCC solutions in adversarial applications.
For example, when considering intrusion detection
\cite<see, e.g.,>{LazarevicEKOS03},
it is necessary to assume that the ``attacking distribution'' has some worst-case characteristics
and it is important to quantify precisely what the adversary knows or can do.
The simple observation in this setting is that an \emph{omniscient and unconstrained adversary}, who knows all
parameters of the game including the learner's strategy, would completely
demolish the learner who uses hard strategies.
By using a soft strategy, the learner can achieve the slightly better result of $1-\delta$
type II error (false negative rate). In either case, the presence of such a powerful adversary makes
the SCC problem trivial and the resulting rejection function is practically worthless.
These simple results are developed in Section~\ref{sec:tooPowerfull}.

We therefore consider an omniscient but limited adversary.
In seeking a useful and quantifiable constraint on $\cQ$ it is
helpful to recall that the essence of the SCC problem is to try to distinguish between
two probability distributions (albeit one of them unknown).
A natural constraint is a lower bound on the ``distance'' between these distributions.
Indeed, it is immediately obvious that if $P \in \cQ$, the adversary can always
achieve the maximal type II error of $1 - \delta$ by selecting $Q = P$.
Following similar results in hypothesis testing
\cite<see>[Chapt. 12]{CoverT91},
we could consider games in which the adversary must select $Q$ such that
$D(P||Q) \geq \Lambda$, for some constant $\Lambda > 0$, where $D(\cdot||\cdot)$ is the KL-divergence;
that is, $D(P||Q) \eqdef \sumn \klpart{p_i}{q_i}$ \cite{CoverT91}.
Unfortunately, this constraint is vacuous since
$D(P||Q)$
``explodes'' when
$q_i \ll p_i$ (for any $i$).
In this case
the adversary
can optimally play the same strategy as in the unrestricted game while meeting the KL-divergence
constraint.
Fortunately, by taking $D(Q||P) \geq \Lambda$, we can effectively constrain the adversary.\footnote{Under
the investment analogy, requiring that $D(P||Q)$ be large is equivalent to requiring a small ``average" value
for $\frac{q_i}{p_i}$ (giving the learner poor investment opportunities). On the other hand, requiring that $D(Q||P)$
be large is equivalent to requiring that the ``average'' value of $\frac{q_i}{p_i}$ be sufficiently large
(providing the learner with good investment opportunities, and potentially increasing the value of $\rho(r,Q)$).}
Instead of only considering the KL-divergence we consider adversary constraints
using a large family of divergences that include the KL-divergence, the $L_2$ norm and various
Bregman divergences. Definitions~\ref{def:twoSymmetric} and~\ref{def:receding} characterize this family.

One of our main contributions is a complete analysis of
this constrained game in Section~\ref{sec:adversaryConstrained},
including identification of the optimal strategy for the learner and the adversary,
as well as the best achievable false negative rate.
The optimal learner strategy  and best achievable rate
are obtained via a solution of a linear program
specified in terms of the problem parameters.
These results are immediately applicable as \emph{lower bounds} for standard (finite-sample)
SCC problems, but may also be used to
inspire new types of algorithms for standard SCC.
While we do not have a closed form expression for the best achievable false-negative rate,
we provide a few numerical examples demonstrating and comparing
the optimal ``hard'' and ``soft'' performance.

\subsection{Unrestricted Adversary}
\label{sec:tooPowerfull}
In the first game we analyze an adversary who is completely unrestricted. This means that
$\cQ$ is the set of all distributions. Unsurprisingly, this game leaves little opportunity
for the learner. For any rejection function $r(\cdot)$, define $r_{min} \eqdef \min_i r(i)$ and
 $I_{min}(r) \eqdef \{i: r(i) = r_{min}\}$.
For any distribution $D$,
$\rho(r,D) = \sumn{d_i r(i)} \geq \sumn{d_i r_{min}} = r_{min}$, in particular,
$\delta = \rho(r,P) \geq r_{min}$ and $\min_Q \rho(r,Q) \geq r_{min}$.
By choosing $Q$ such that $q_i = 1$ for some  $i \in I_{min}(r)$,
the adversary can achieve $\rho(r,Q) = r_{min}$
(the same rejection rate is achieved by taking any $Q$ with
$q_i = 0$ for all $i \not\in I_{min}(r)$).
In the soft setting, $\min_Q \rho(r,Q)$ is maximized by the rejection function
$r^\delta(i) \eqdef \delta$ for all $p_i > 0$ ($r^\delta(i) \eqdef 1$ for all $p_i = 0$).
This is equivalent to flipping a $\delta$-biased coin for non-null events (under $P$).
The best achievable type II error is $1 - \delta$. In the hard setting,
clearly $r_{min} = 0$ (otherwise $1 > \delta \geq 1$), and the best  achievable type II error is
precisely 1. That is, absolutely nothing can be achieved.

This simple analysis shows the futility of the SCC game when the adversary is too powerful.
In order to consider SCC problems at all one must consider reasonable restrictions on the adversary
that lead to more useful games. One type of such a restriction would be to limit the adversary's knowledge of
$r(\cdot)$, $P$ and/or of $\delta$.
Another type would be to directly limit the strategic choices
available to the adversary.
We note that the former type of restriction doesn't affect the best achievable
type II error, and thus in the next section we will focus on the latter.

\subsection{An Omniscient, but Constrained, Adversary}
\label{sec:adversaryConstrained}
While we could therefore define $\cQ = \cQ_{\Lambda} \eqdef \{ Q :  D(Q||P) \geq \Lambda \}$,
we instead will consider a more general family. First,
let $\cX$ be the $N$-dimensional simplex:
$\cX \eqdef \{(x_1, \dots, x_N) : x_i \geq 0, \sumn x_i = 1 \}$.
For convenience, we now define a transfer function, $t(X,a,b) \to \cX$, where
$X \in \cX$, and $a$ and $b$ are indices in $\{1, \dots, N\}$, which transfers probability from event $b$ to event $a$, as:
\begin{align*}
t(X,a,b)_i = \begin{cases}
    x_a + x_b& i=a,\\
    0& i=b,\\
    x_i& \text{otherwise} .
    \end{cases}
\end{align*}

\begin{definition}[2-Symmetric]
\label{def:twoSymmetric}
A function $f_P: \cX \to \Real$, is called \emph{2-symmetric}
if for all $X \in \cX$
and for all $j,k$ such that $p_j = p_k$, $f_P\left(t(X,j,k)\right) = f_P\left(t(X,k,j)\right)$.
\end{definition}
\begin{remark}
\label{rem:bregman1}
We note that a Bregman divergence (defined over $[0,1]^N$) may be 2-symmetric. Specifically,
define $D_P(Q) = B_F(Q||P) \eqdef F(Q) - F(P) - \nabla F(P) \cdot (Q - P)$.
Let $\Delta_F \eqdef F(t(Q,j,k)) - F(t(Q,k,j))$.
Then, the divergence is 2-symmetric if:
\begin{align*}
0 = D_P(t(Q,j,k)) - D_P(t(Q,k,j)) =& \Delta_F - \nabla F(P) \cdot \left(t(Q,j,k) - t(Q,k,j)\right)\\
  =& \Delta_F + (q_j + q_k)\left(\pderiv{F(P)}{x_k} - \pderiv{F(P)}{x_j}\right).
\end{align*}
We note that if $F(X) = \sumn f(x_i)$, where $f(\cdot)$ is a strictly convex function,
then clearly the Bregman divergence is 2-symmetric.
\end{remark}

\begin{definition}[Receding]
\label{def:receding}
A function $f_P: \cX \to \Real$, is called \emph{receding}
if for all $X \in \cX$, $p_j < p_k$ and $x_k > 0$, $f_P(t(X,j,k)) > f_P(X)$.
A receding function $D_P: \cX \to \Real$ is called a \emph{receding divergence} if
it is defined over the domain $[0,1]^N$, it is differentiable over $(0,1)^N$ and is strictly convex.
\end{definition}

\begin{remark}
\label{rem:bregman2}
We note that a Bregman divergence may be a receding divergence, as well. Specifically,
define $D_P(Q) = B_F(Q||P) \eqdef F(Q) - F(P) - \nabla F(P) \cdot (Q - P)$. This trivially
meets the differentiability and strict convexity requirements.
Let us examine if it is receding.
Let $p_j < p_k$, $q_k > 0$ and let $\Delta \eqdef t(Q,j,k) - Q$. Then, in order to satisfy the property:
\begin{align*}
0 < D_P(t(Q,j,k)) - D_P(Q) =& F(Q + \Delta) - F(Q) - \nabla F(P) \cdot \Delta\\
=& F(Q + \Delta) - F(Q) +
q_k\left(\pderiv{F(P)}{x_k} - \pderiv{F(P)}{x_j}\right).
\end{align*}
We note that if $F(X) = \sumn f(x_i)$, where $f(\cdot)$ is a strictly convex function,
then $F(t(X,j,k)) = F(t(X,k,j))$ for all $j,k$, and thus, by convexity:
\begin{align*}
F(Q + \Delta) - F(Q) = F(t(Q,j,k)) - F(Q) \geq 0\\
\pderiv{F(P)}{x_k} - \pderiv{F(P)}{x_j} = f'(p_k) - f'(p_j) > 0.
\end{align*}
Thus, Bregman divergences which are of this form, such as the squared Euclidean distance $D_P(Q) = ||Q - P||^2$
and the KL-Divergence, are also (2-symmetric) receding divergences. Note that this condition is sufficient and not necessary.
It is certainly possible for Bregman divergences which are not of this form to be receding divergences as well.
\end{remark}

We define $\cQ_\Lambda \eqdef \{Q : D_P(Q) \geq \Lambda\}$, where $D_P(\cdot)$ is a 2-symmetric receding divergence.
We say that a distribution $Q$ \emph{meets the divergence constraint} if $D_P(Q) \geq \Lambda$.
As we will shortly see, this is consistent with an adversary that can't eavesdrop on the user,
as the constraint prevents the adversary from selecting distributions which are only concentrated on high-probability events under $P$.
\begin{lemma}
\label{lem:gd_hasPropsABC}
$\cQ_{\Lambda}$ possesses Properties $A$ and $B$ w.r.t.~$P$.
\end{lemma}
\begin{proof}
Let $j,k$ be such that $p_j \leq p_k$. For any distribution $Q \in \cQ_{\Lambda}$ we define $Q' = t(Q,j,k)$.
If $p_j < p_k$, then since $D_P(\cdot)$ is receding, $D_P(Q') \geq D_P(Q) \geq \Lambda$.
Otherwise, if $p_j = p_k$, since $D_P(\cdot)$ is 2-symmetric and convex, $D_P(Q') \geq D_P(Q) \geq \Lambda$.
Thus, in either case, $Q' \in \cQ_{\Lambda}$.
If $Q$ is such that $q_j < q_k$, then $q'_j + q_j = 2q_j + q_k \geq q_k = q'_k + q_k$, and
$\cQ_{\Lambda}$ has $\PropA$. If $Q$ is such that $\frac{q_j}{p_j} < \frac{q_k}{p_k}$, then
$\frac{q'_j}{p_j} = \frac{q_j+q_k}{p_j} \geq 0 = \frac{q'_k}{p_k}$ and $\cQ_{\Lambda}$ possesses
$\PropB$.
\end{proof}
Therefore, by Theorems~\ref{thm:hard-decision}
and~\ref{thm:chr-monotonic-equality}
there exists a (strictly) monotone $r \in \cRdeltastar$ in the hard (respectively, soft)
setting. If $Q_\Lambda$ has $\PropC$ as well, then by Theorem~\ref{thm:LDRS_optimal} any $\delta$-valid
LDRF is hard-optimal. It is easy to verify that Bregman divergences of the form described in Remark~\ref{rem:bregman2}
possess $\PropC$.

We now define $X^{(j)}$ as the distribution which is completely concentrated on event $j$. In
other words $x^{(j)}_i \eqdef \ind(i=j)$, where $\ind(\cdot)$ is the indicator function.
We assume that $0 < p_1 \leq p_2 \leq \dots \leq p_N$. Therefore, since $D_P(\cdot)$ is receding,
$D_P\left(X^{(1)}\right) \geq D_P\left(X^{(2)}\right) \geq \dots \geq D_P\left(X^{(N)}\right)$.
Therefore if $D_P\left(X^{(N)}\right) \geq \Lambda$, then any $Q$ that is concentrated on a single event
meets the constraint $D_P(Q) \geq \Lambda$. Then, the adversary can play the same strategy as in
the unrestricted game, and the learner should select $r^\delta$ as before.
For the game to be non-trivial it is thus required that $\Lambda > D_P\left(X^{(N)}\right)$.
Similarly, if the optimal $r$ is such that there exists $j \in I_{min}(r)$ (that is $r(j) = r_{min}$) and
$D_P\left(X^{(j)}\right) \geq \Lambda$, then a distribution $Q$ that is completely concentrated on $j$
has $D_P(Q) \geq \Lambda$ and achieves $\rho(r,Q) = r_{min}$, as in the unrestricted game.
Therefore, $r = r^\delta$, and so maximizes $r_{min}$. This yields the following definition:
\begin{definition}
A rejection function $r$ is called \emph{vulnerable} if there exists $j \in I_{min}(r)$
such that $D_P\left(X^{(j)}\right) \geq \Lambda$.
\end{definition}

We begin our analysis of the game by identifying some useful characteristics of
optimal adversary strategies under the assumption that the chosen rejection function isn't vulnerable.
These properties, that are stated in Lemma~\ref{lem:nature-gd-characteristics},
are then used to prove Theorem~\ref{thm:gd-2-points} showing that the effective support of an optimal
$Q$ has a size of two at most. Based on these properties, we provide in Theorem~\ref{thm:linear_program}
a linear program that computes an optimal rejection function (under the assumption that it
isn't vulnerable).Finally, in Lemma~\ref{lem:lin-prog-delta}
we show that the solution computed by the linear program is $r^\delta$ if it is vulnerable,
giving optimal (though trivial) performance. Thus, in any case, the output of the linear program is optimal.

If $\Lambda > D_P\left(X^{(1)}\right)$, then no adversary distribution can meet
the divergence constraint. We therefore
limit ourselves to cases where $\Lambda \leq D_P\left(X^{(1)}\right)$.
We can now divide
the events in $\Omega$ into two groups: $H$ and $L$, such that $H = \{ i : D_P\left(X^{(i)}\right) < \Lambda\}$ and
$L = \Omega \setminus H$.
We note that the assumption that $r$ isn't vulnerable implies that $I_{min}(r) \subseteq H$.
By definition, $\forall h \in H,l \in L$, we have that $p_h > p_l$.

\begin{lemma}
\label{lem:prob-in-Lo}
If $Q$ meets the divergence constraint, there exists an event $i \in L$ for which $q_i > 0$.
\end{lemma}
\begin{proof}
Let us assume that $q_i = 0$ for all $i \in L$. Let $j$ be the smallest event in $H$. Since
$D_P(\cdot)$ is receding, $D_P(Q) \leq D_P\left(X^{(j)}\right) < \Lambda$. Contradiction.
\end{proof}

\begin{lemma}
\label{lem:nature-gd-characteristics}
Let $r$ be a rejection function which isn't vulnerable. If $Q$ meets the divergence constraint and minimizes $\rho(r,Q')$:
\renewcommand{\theenumi}{\roman{enumi}}
\begin{enumerate}
    \item $D_P(Q) = \Lambda$;
    \item Let $u,v$ be two indices in $\{1, \dots, N\}$. Define $Q'' = t(Q,u,v)$.
    If $q_v > 0$ and $D_P(Q'') \geq \Lambda$, then $r(u) \geq r(v)$. Furthermore, $r(u) = r(v) \Rightarrow D_P(Q'') = \Lambda$;
    \item $p_j < p_k$ and $q_k > 0 \Rightarrow r(j) > r(k)$;
    \item $p_j < p_k$ and $q_j, q_k > 0 \Rightarrow \pderiv{D_P(Q)}{x_j} > \pderiv{D_P(Q)}{x_k}$;
    \item $q_j, q_k > 0 \Rightarrow p_j \neq p_k$;
    \item $p_j < p_k$ and $q_j > 0 \Rightarrow D_P(Q) > D_P(t(Q,k,j))$.
\end{enumerate}
\end{lemma}
\begin{proof}
\renewcommand{\theenumi}{\roman{enumi}}
\begin{enumerate}
  \item Assume that $D_P(Q) > \Lambda$. By Lemma~\ref{lem:prob-in-Lo}
   there exists a non-empty set $L_Q \eqdef \{ i \in L \; | \; q_i > 0\}$.
  Let $h_{max} = \argmax_{i \in I_{min}(r)}q_i$. Clearly, $h_{max} \in H$. We define a new distribution $Q^*$,
  which is identical to $Q$ except that probability is transferred from events in $L_Q$ to $h_{max}$, in order to make
  $D_P(Q^*) = \Lambda$ (this is possible, since $D_P(\cdot)$ is continuous and, by Lemma~\ref{lem:prob-in-Lo},
  transferring all probability
  from $L_Q$ to $h_{max}$ would result in $D_P(\cdot) < \Lambda$). Since transferring any probability from $i \in L_Q$ to
  $h_{max}$ results in making $\rho(r,Q)$ smaller, $\rho(r,Q^*) < \rho(r,Q)$, contradicting
  the fact that $Q$ minimizes $\rho(r,Q')$.

  \item We note that $\rho(r,Q'') = \rho(r,Q) - q_v (r(v) - r(u))$. Since $\rho(r,Q)$ is minimal
  and $D_P(Q'') \geq \Lambda$
  it follows that $r(u) \geq r(v)$. If $r(u) = r(v)$ then $\rho(r,Q'') = \rho(r,Q)$,
  and by part (i), $D_P(Q'') = \Lambda$.

  \item By part (ii), taking $u=j$ and $v=k$ we trivially get $r(j) \geq r(k)$.
  Furthermore, since $p_u = p_j < p_k = p_v \Rightarrow D_P(Q'') > \Lambda$, $r(j) \neq r(k)$. Thus, $r(j) > r(k)$.

  \item
Assume, contradictorily,  that $\pderiv{D_P(Q)}{x_j} \leq \pderiv{D_P(Q)}{x_k}$.
Let $0 < \epsilon \leq \min\{q_j, q_k\}$. We define $\epsilon_{j,k} = \epsilon\left(X^{(k)} - X^{(j)}\right)$. Then, by convexity:
\begin{align*}
D_P(Q + \epsilon_{j,k}) \geq& \ D_P(Q) + \nabla D_P(Q) \cdot \epsilon_{j,k}\\
    =& \ D_P(Q) + \epsilon\left(\pderiv{D_P(Q)}{x_k} - \pderiv{D_P(Q)}{x_j}\right)\\
    \geq& \ D_P(Q).
\end{align*}
Therefore, by defining $Q' = Q + \epsilon_{j,k}$, we have that $D_P(Q') \geq D_P(Q) \geq \Lambda$.
  Furthermore, by part (iii), $r(j) > r(k)$. Therefore,
  $\rho(r,Q') = \rho(r,Q) + \epsilon(r(k) - r(j)) < \rho(r,Q)$. Contradiction.

  \item Assume that $p_j = p_k$. We consider two cases. In the first case, $r(j) < r(k)$,
        w.l.o.g.
        By defining $u=j$, $v=k$, from part (ii) we get that $r(j) \geq r(k)$, which is a contradiction.
        In the second case, $r(j) = r(k)$. However, since both $q_j$ and $q_k$ are greater
        than zero, defining $u=j$ and $v=k$ in part (ii) gives us that $D_P(Q'') > \Lambda$, which
        is again a contradiction.

  \item  If $q_k = 0$ then $Q = t(t(Q,k,j),j,k)$ and $D_P(Q) > D_P(t(Q,k,j))$. Otherwise, $q_k > 0$
  and by part (iii), $r(j) > r(k)$. If we assume contradictorily that $D_P(t(Q,k,j)) \geq D_P(Q) = \Lambda$,
  then by part (ii), taking $u = k$ and $v = j$, $r(k) \geq r(j)$. Contradiction.

\end{enumerate}
\end{proof}

\begin{theorem}
\label{thm:gd-2-points}
If $r$ isn't vulnerable, then any optimal adversarial strategy  $Q$ has an effective support of size at most two.
\end{theorem}
\begin{proof}
Let us assume, by contradiction, that the theorem's statement is wrong; that is, there exists an optimal $Q^*$
that has $J > 2$ events for which $q^*_i \neq 0$. W.l.o.g. we rename our events such that
these are the first $J$ events. We note that $Q^*$ is a
solution (i.e., global minimum) to the following problem ($*$):
\begin{align*}
  \text{minimize}\;\rho(r,Q) = \sumj{r(i)q_i}, \;&\; \text{subject to:}\\
                 \sumj{q_i} = 1, \;&\; D_P(Q) = \Lambda,\\
                 0 < q_i < 1, \;&\; i \in \{1, \dots, J\}.
\end{align*}
We will now prove that $Q^*$ does not in fact solve the problem. We do so in two parts:
\begin{enumerate}
  \item We show that $Q^*$ is the unique global maximum of the Lagrangian of ($*$).
  \item We show that there exists a different distribution $\tilde{Q}$ with the same effective support, which
  meets the equality constraints. We therefore conclude that $\rho(r,\tilde{Q}) < \rho(r,Q)$, contradicting the
  optimality of $Q^*$.
\end{enumerate}
We now prove the first part. The Jacobian matrix for the equality constraints at $Q^*$ is:
$$
\begin{pmatrix}
  {1} & {1} & {1} & {\dots} & {1} \\
  {\pderiv{D_P(Q^*)}{x_1}} & {\pderiv{D_P(Q^*)}{x_2}} & {\pderiv{D_P(Q^*)}{x_3}} & {\dots} & {\pderiv{D_P(Q^*)}{x_J}}
\end{pmatrix}.
$$
Since all $q^*_i > 0$, by parts (v) and (iv) of Lemma~\ref{lem:nature-gd-characteristics}, for all $j,k \leq J$:
$p_j \neq p_k$ and $\pderiv{D_P(Q^*)}{x_j} \neq \pderiv{D_P(Q^*)}{x_k}$.
Therefore, the gradients of the constraints are linearly
independent at $Q^*$ and therefore, since $Q^*$ is (at least) a local minimum to the problem $(*)$, there exists a unique
Lagrangian multiplier vector $\lambda = (\lambda_1, \lambda_2)$ such that
$Q^* = (q^*_1, q^*_2, \dots, q^*_J)$ is an extremum point of the Lagrangian:
$$
L(Q,\lambda) = \sumj{r(i)q_i} + \lambda_1 \left(D_P(Q) - \Lambda\right) + \lambda_2 \left(\sumj{q_i} - 1\right).
$$
The partial derivatives are:
$\pderiv{L(Q^*,\lambda)}{q_i} = r(i) + \lambda_1\pderiv{D_P(Q^*)}{x_i} + \lambda_2 = 0$.
Therefore, for all $j,k \in \{1, \dots, J\}$:
\begin{align*}
r(j) + \pderiv{D_P(Q^*)}{x_j} + \lambda_2 &= r(k) + \lambda_1\pderiv{D_P(Q^*)}{x_k} + \lambda_2\\
\Rightarrow\;\lambda_1\left(\pderiv{D_P(Q^*)}{x_j} - \pderiv{D_P(Q^*)}{x_k}\right)&= r(k) - r(j)\\
\Rightarrow\;\lambda_1 &= \frac{r(k) - r(j)}{\pderiv{D_P(Q^*)}{x_j} - \pderiv{D_P(Q^*)}{x_k}}
\end{align*}
If we assume (w.l.o.g.) that $p_k < p_j$, then, from parts (iii) and (iv) of Lemma~\ref{lem:nature-gd-characteristics},
$r(k) > r(j)$ and $\pderiv{D_P(Q^*)}{x_k} > \pderiv{D_P(Q^*)}{x_j}$. Thus, $\lambda_1 < 0$. Therefore,
due to the strict convexity of $D_P(\cdot)$ and the linearity of the other two equations, the Lagrangian $L(Q,\lambda)$
is strictly concave.
Therefore, since $Q^*$ is an extremum point of the (strictly concave) Lagrangian function, it is the unique global maximum.

We now wish to show that there exists some other distribution $\tilde{Q}$ that meets the divergence constraint
and has the same support as $Q^*$. We define $Q^{123}$ as $q^{123}_i = \ind(i > 3)q^*_i$
and $c_{123} \eqdef q^*_1 + q^*_2 + q^*_3$. Then we define:
\begin{align*}
g(q_1,q_2) \eqdef \ \ & Q^{123} + q_1X^{(1)} + q_2X^{(2)} + (c_{123} - q_1 - q_2)X^{(3)}\\
f(q_1,q_2) \eqdef \ \ & D_P\left(g(q_1,q_2)\right) - \Lambda\\
\Rightarrow \textrm{for } i \in \{1,2\}:\ \pderiv{f(q_1,q_2)}{q_i} =& \nabla D_P\left(g(q_1,q_2)\right) \cdot \left(X^{(i)} - X^{(3)}\right)\\
 =& \pderiv{D_P(g(q_1,q_2))}{x_i} - \pderiv{D_P(g(q_1,q_2))}{x_3}
\end{align*}
 Clearly, $g(q^*_1,q^*_2) = Q^*$ and $f(q^*_1,q^*_2) = 0$. From part (iv) of Lemma~\ref{lem:nature-gd-characteristics},
 we have for $i \in \{1, 2\}$:
$$
           \pderiv{f(q^*_1,q^*_2)}{q_i} = \pderiv{D_P(Q^*)}{x_i} - \pderiv{D_P(Q^*)}{x_3} \neq 0.
$$
Therefore, $f$ is smooth in the open, convex domain $\{q_1, q_2 > 0\} \cap \{q_1 + q_2 < c_{123}\}$
and has a root in this domain at $(q^*_1, q^*_2)$ at which none of its partial derivatives are 0.
Then, there exist an infinite number of points in the domain for which $f = 0$ (this is true for
any sub-domain for which $(q^*_1, q^*_2)$ is an interior point).
Let $(\tilde{q}_1, \tilde{q}_2) \neq (q^*_1, q^*_2)$ be one of these points. Then, the distribution
$\tilde{Q} = (\tilde{q}_1, \tilde{q}_2, c_{123} - \tilde{q}_1 - \tilde{q}_2, q^*_4, q^*_5, \dots, q^*_J) \neq Q^*$
satisfies $D(\tilde{Q},P) = \Lambda$ and has the exact same effective support as $Q^*$. Therefore,
$\tilde{Q}$ meets the equality criteria of the Lagrangian. Since $Q^*$ is the unique global maximum
of $L(Q, \lambda)$:
$\rho(r,\tilde{Q}) = L(\tilde{Q}, \lambda) < L(Q^*, \lambda) = \rho(r,Q^*)$,
contradicting the fact that $Q^*$ is optimal.
\end{proof}

We now turn our attention to the learner's selection of $r(\cdot)$. As already established
by Lemma~\ref{lem:gd_hasPropsABC} and Theorem~\ref{thm:chr-monotonic-equality}, it is sufficient
for the learner to consider only strictly monotone rejection functions. Since for these functions
$p_j = p_k \Rightarrow r(j) = r(k)$, the learner can partition $\Omega$
into $K = K(P, \Omega)$ event subsets, which correspond, by probability, to ``level sets'', $S_1, S_2, \dots, S_K$
(all events in a level set $S_j$
have probability $p^{(S_j)}$). We re-index these subsets such that $0 < p^{(S_1)} < p^{(S_2)} < \cdots < p^{(S_K)}$.
Define $K$ variables $r_1, r_2, \dots, r_K$,
representing the rejection rate assigned to each of the $K$ level sets ($\forall \omega \in S_i, r(\omega) = r_i$).
Since $D_P(\cdot)$ is 2-symmetric, $D_P\left(X^{(\omega)}\right)$ is constant for all $\omega$ in a level set $S$.
Therefore, we use the notation $D^S_P \eqdef D_P\left(X^{(\omega)}\right)$ for any $\omega \in S$.
We group our level sets by probability: $\cL = \{ S : D^S_P > \Lambda \}$, $\cM = \{ S : D^S_P = \Lambda \}$,
and $\cH = \{ S : D^S_P < \Lambda \}$. We define $w \eqdef \argmax_i \{S_i \in \cL \bigcup \cM\}$.
\begin{lemma}
\label{lem: w}
If $Q$ minimizes $\rho(r,Q)$ and meets the constraint $D_P(Q) \geq \Lambda$, then
$r_w \geq \rho(r,Q)$.
\end{lemma}
\begin{proof}
Let $j \in S_w$. Then $D_P\left(X^{(j)}\right) \geq \Lambda$, and since $Q$ minimizes $\rho(r,Q)$,
$r_w = \rho\left(r,X^{(j)}\right) \geq \rho(r,Q)$.
\end{proof}

By Theorem~\ref{thm:gd-2-points}, if $r$ isn't vulnerable, the adversary-optimal $Q$ will have an effective support of
at most size 2. If it has an
effective support of size 1, then the event $\omega$ for which $q_{\omega} = 1$ cannot
 be from a level set in $\cL$ or $\cH$
(otherwise, part (i) of Lemma~\ref{lem:nature-gd-characteristics} would be violated).
Therefore, it must belong to the single level set in $\cM$. Thus, if $\cM = \{S_m\}$
(for some index $m$),
there are feasible solutions $Q$ such that $q_\omega = 1$ (for $\omega \in S_m$), all of which have
$\rho(r,Q) = r_m$. The following lemma characterizes optimal distributions $Q$ which have an effective support
of size 2.
\begin{lemma}
If $r$ isn't vulnerable and
$Q$ is optimal with an effective support of size 2 (that is, there are $j,k$ such that $q_j,q_k > 0$ and
$q_j + q_k = 1$), then, assuming w.l.o.g. that $p_j \leq p_k$,
$j \in S_l \in \cL$ and $k \in S_h \in \cH$ for some $l$ and $h$.
\end{lemma}
\begin{proof}
Since $q_j, q_k > 0$, and $Q$ is optimal,
we have that $p_j \neq p_k$, by part (v) of Lemma~\ref{lem:nature-gd-characteristics}.
Therefore, $p_j < p_k$, and by part (vi) of Lemma~\ref{lem:nature-gd-characteristics},
$$
D_P\left(X^{(k)}\right) = D_P(t(Q,k,j)) < D_P(Q) < D_P(t(Q,j,k)) = D_P\left(X^{(j)}\right).
$$

Assume, by contradiction, that $k$ belongs to a level set in $\cL$ or $\cM$.
This is equivalent to $D_P\left(X^{(k)}\right) \geq \Lambda$.
We therefore have that $D_P(Q) > D_P\left(X^{(k)}\right) \geq \Lambda$, which is a violation of part (i) of
Lemma~\ref{lem:nature-gd-characteristics}. Therefore, $k$ belongs to a level set in $\cH$. Likewise, were
we to assume that $j$ belongs to a level set in $\cM$ or $\cH$ ($D_P\left(X^{(j)}\right) \leq \Lambda$),
it would follow that $D_P(Q) < D_P\left(X^{(j)}\right) \leq \Lambda$, which would also violate part (i)
of Lemma~\ref{lem:nature-gd-characteristics}. Therefore, $j$ belongs to a level set in $\cL$.
\end{proof}
\begin{lemma}
\label{lem:qlh_soltn}
Let $S_l \in \cL$ and $S_h \in \cH$. Then, there always exists a single solution $q_\Lambda^{(l,h)} \in (0,1)$ to
$$D_P\left(qX^{(j)} + (1-q)X^{(k)}\right) = \Lambda,$$ for any $j \in S_l, k \in S_h $.
\end{lemma}
\begin{proof}
Let $Q$ be a distribution with an effective support of size 2, where the events $j,k$ for which
$q_j, q_k > 0$ are such that $j \in S_l$ and $k \in S_h$. Furthermore, let $q_j = q$ and
$q_k = 1-q$.
Define $g(q) \eqdef g(q,j,k) \eqdef D_P\left(qX^{(j)} + (1-q)X^{(k)}\right)$.
Then, $g(q) = D_P(Q)$.
We note that $g(0) = D^{S_h}_P < \Lambda$
and $g(1) = D^{S_l}_P > \Lambda$. Thus a solution, $q^*$, exists in the range
$(0,1)$. Since $g(q)$ is continuous and convex, there cannot exist another solution in this range.
Let $X = q^*X^{(j)} + (1-q^*)X^{(k)}$. Let $j' \in S_l$ and $k' \in S_h$.
Then, since $D_P(\cdot)$ is 2-symmetric,
$\Lambda = D_P\left(X\right) = D_P\left(t(X,j',j)\right) = D_P\left(t(X,k',k)\right) = D_P\left(t\left(t(X,j',j),k',k\right)\right)$, and
thus the solution is the same for all pairs of members between $S_l$ and $S_h$.
\end{proof}

Therefore, if an adversary-optimal $Q$ has an effective support of size 2,
where the events with non-zero probability are from $S_l$ and $S_h$ respectively,
then, $\rho(r,Q) = \rho^{(l,h)} \eqdef q_\Lambda^{(l,h)} r_l + (1 - q_\Lambda^{(l,h)}) r_h$.

Therefore, the adversary's choice of an optimal distribution, $Q$, must have
one of $|\cL||\cH| + |\cM| \leq \lfloor \frac{K^2}{4} \rfloor$
(possibly different) rejection rates.
Each of these rates, $\rho_1, \rho_2, \dots, \rho_{|\cL||\cH| + |\cM|}$,
is a linear combination of at most two variables, $r_i$ and $r_j$.
We introduce an additional variable, $z$, to represent the max-min rejection rate.
This entails the following theorem.
\begin{theorem}
\label{thm:linear_program}
An optimal soft rejection function
and the lower-bound on the optimal type II error, $1-z$, is obtained by solving the following
linear program:
\begin{align}
\label{eqn:linprog}
  \text{maximize}_{r_1, r_2, \dots, r_K, z} \;\;z, \;&\; \text{subject to:}\nonumber\\
                 &\sum_{i=1}^{K}{r_i|S_i|p(S_i)} = \delta \\
                 &1 \geq r_1 \geq r_2 \geq \dots \geq r_K \geq 0\nonumber\\
                 &r_w \geq z\nonumber\\
                 &\rho_i \geq z, \; i \in \{1, 2, \dots, |\cL||\cH|+|\cM|\} .\nonumber
\end{align}
\end{theorem}
Let $r^*$ be the solution to the linear program~(\ref{eqn:linprog}). Our derivation of the linear program
is dependent on the restriction that $r^*$ isn't vulnerable.
If $r^*$ contradicts this restriction then, as discussed,
the optimal strategy is $r^\delta$. The following lemma shows that in this case $r^* = r^\delta$ anyway, and thus the solution to the linear program is always optimal.
Its proof can be found in Appendix~\ref{sec: appendix-games-proof}.
\begin{lemma}
\label{lem:lin-prog-delta}
Let $r^*$ be the solution to the linear program. If $r^*$ is vulnerable, then $r^* = r^\delta$.
\end{lemma}

\begin{remark}
We attempted to determine explicit bounds on the value of $1 - z$, the optimal type II error, that
would result from solving the linear program in Theorem~\ref{thm:linear_program}, including via examining
the dual form of the problem, but were unsuccessful. If the optimal rejection
function $r^* \neq r^\delta$ then one can prove several interesting properties, some of which we have proven
in Lemma~\ref{lem:nature-gd-characteristics}, which may be of use in determining bounds
on the optimal type II error. However, as the following example illustrates,
even determining whether or not the optimal solution outperforms $r^\delta$ is not trivial.
\begin{example}
\label{ex:dkl_not_trivial}
Let $P = \{0.05, 0.05, \cdots, 0.05, 0.2\}$, $\delta = 0.2$, $\Lambda = 3$ and $D_P(\cdot) = D(\cdot||P)$
be the KL-divergence.
Then, solving the linear program gives $r^\delta$ (it is possible that other solutions exist, however).
Interestingly, changing $\delta$ does not appear to change the result (even when taking values
as small as $\delta = 0.001$, or as large as $\delta = 0.999$).
Furthermore, if we increase $\Lambda$ to $3.2$, we achieve solutions to the linear
program which aren't $r^\delta$, but do not improve on its rejection rate (again for
the same range of $\delta$ values).
\end{example}
\end{remark}

\begin{figure}[ht]
\label{fig:performance}
\centerline{
\begin{tabular}{cc}
\psfig{figure=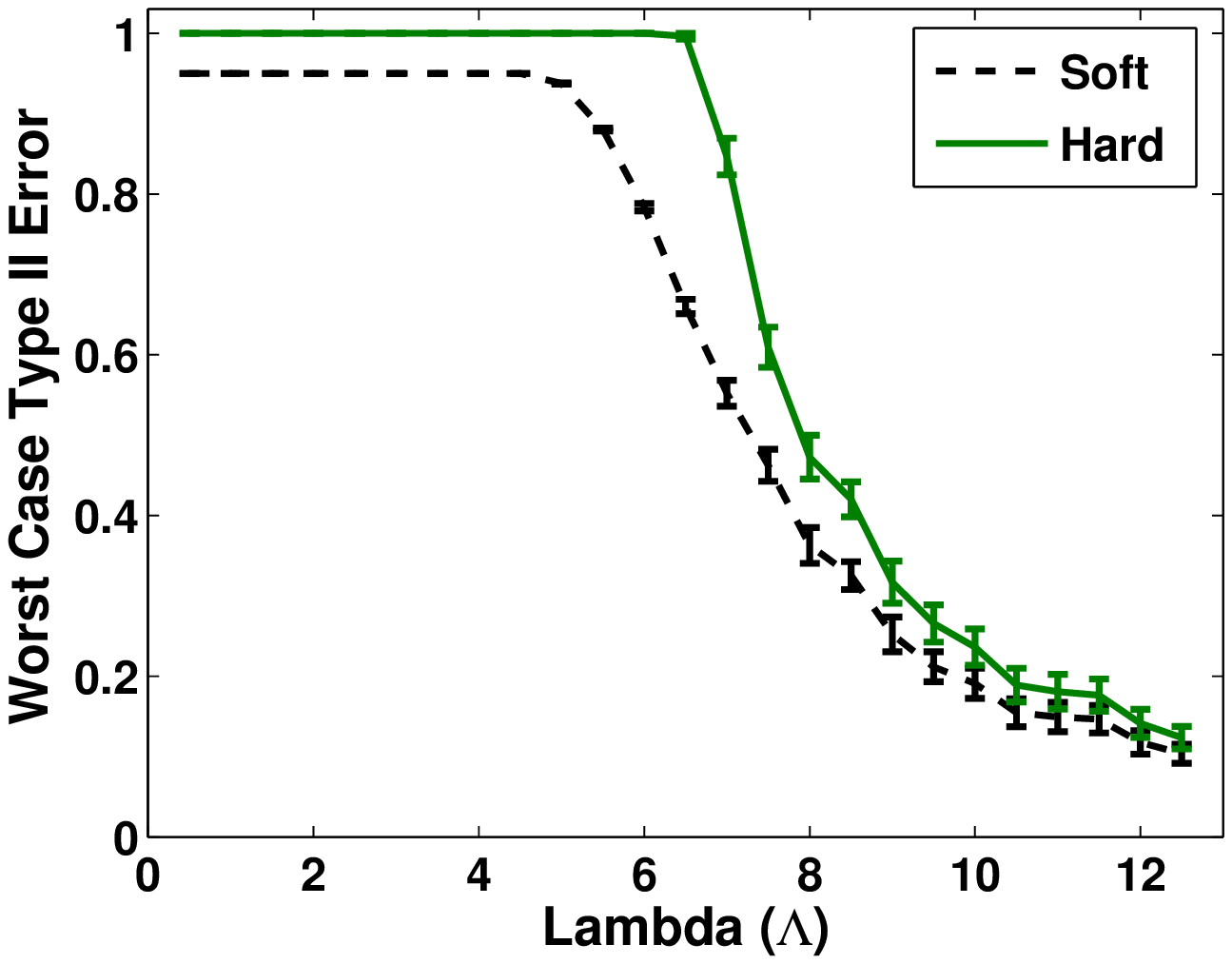,height=6cm}
&
\psfig{figure=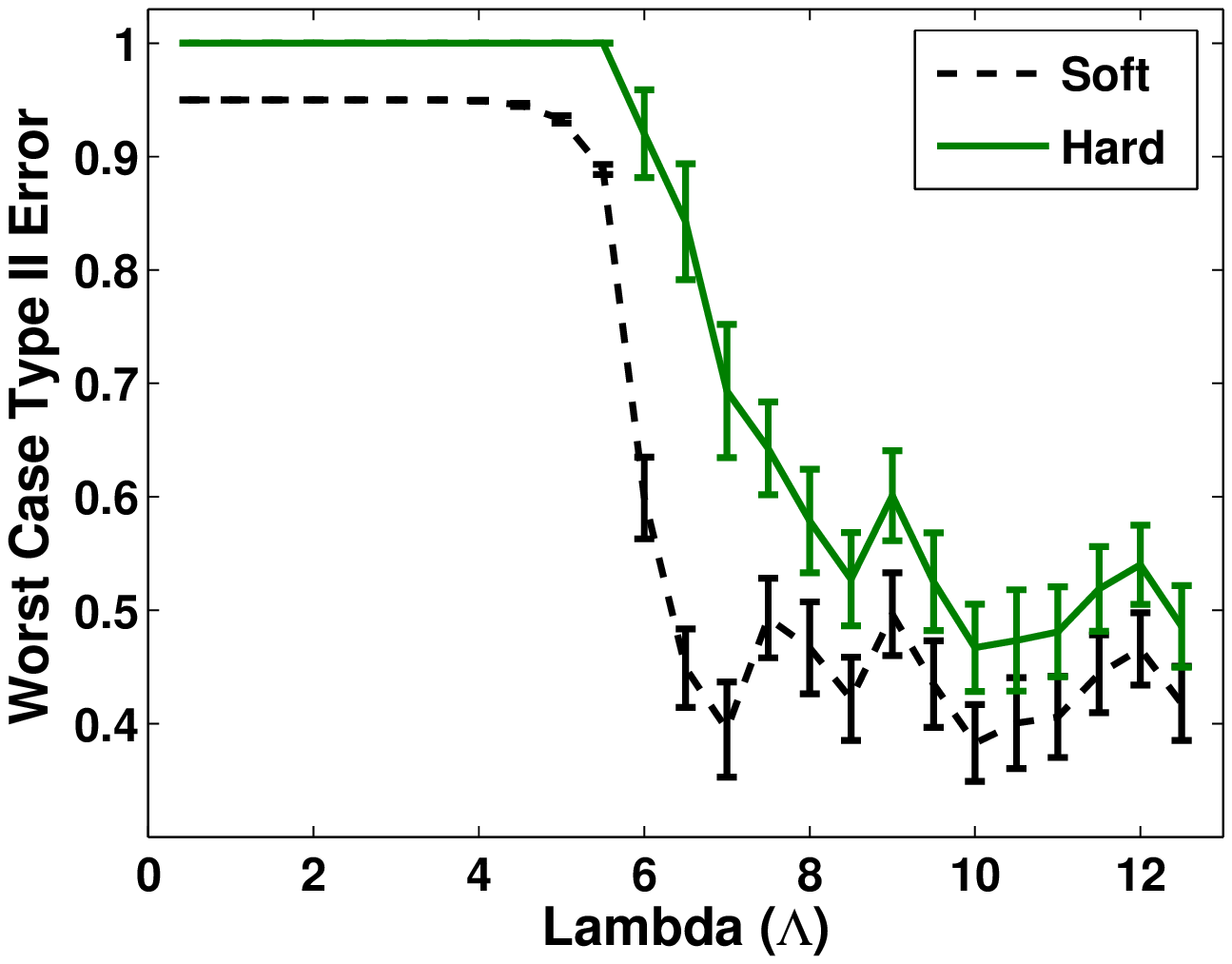,height=6cm}\\
(a) Arbitrary & (b) Gaussians
\end{tabular}
}
\caption{
Type II error vs. $\Lambda$, for $N = 50$ and $\delta = 0.05$. 50 distributions
were generated for each value of $\Lambda$ ($\Lambda = 0.5, 0.1, \cdots, 12.5$).
Error bars depict standard error of the mean (SEM).
}
\end{figure}

\subsubsection{Numerical Examples}
\label{sec:numerical}
We numerically compare the performance of hard and soft rejection strategies for a constrained game,
where $D(Q||P) \geq \Lambda$, for various values of $\Lambda$, and two
different families of target distributions, $P$, over a support of size $N=50$. The families are arbitrary probability
mass functions over $N$ events and discretized Gaussians (over $N$ bins).
For each $\Lambda$ we generated 50 random distributions $P$
for each of the families.
For each such $P$ we solved the optimal hard and soft strategies and computed the corresponding
worst-case optimal type II error, $1-\rho(r,Q)$.

Since $\max_Q D(Q||P) = \log (1 / \min_i p_i)$,
it is necessary that $\min_i p_i \leq 2^{-\Lambda}$
when generating $P$ (to ensure that a $\Lambda$-distant $Q$ exists).
Distributions in the first family of arbitrarily random distributions, Figure~6.1(a),
are generated
by sampling a point ($p_1$) uniformly in $(0,2^{-\Lambda}]$. The other $N-1$ points are drawn i.i.d.
$\sim U(0,1]$, and then normalized so that their sum is $1 - p_1$. The second family,
Figure~6.1(b),
are Gaussians centered at $0$ and discretized over $N$ evenly
spaced bins in the range $[-10,10]$.
A (discretized) random Gaussian $N(0,\sigma)$ is selected by choosing $\sigma$ uniformly in
some range $[\sigma_{min},\sigma_{max}]$.
$\sigma_{min}$ is set to the minimum $\sigma$ ensuring that the first/last bin
will not have ``zero'' probability (due to limited precision).
$\sigma_{max}$ was set so that the cumulative probability in the first/last
bin will be $2^{-\Lambda}$, if possible (otherwise $\sigma_{max}$ is arbitrarily
set to  $10 * \sigma_{min}$).

The results for $\delta = 0.05$ are shown in Figure~6.1.
Other results (not presented) for a wide variety of the problem parameters (e.g., $N$, $\delta$) are qualitatively
the same.
It is evident that both the soft and hard strategies are ineffective for small $\Lambda$.
Clearly, the soft method has significantly lower error than that of the  hard (until $\Lambda$ becomes
``sufficiently large'').


\section{Low Density Rejection in a Continuous Setting}
\label{sec:cont}

In Section~\ref{sec:characterization} we presented a number of results on LDRS optimality in a simplified finite and discrete
setting. In this section, we reconsider LDRS (now only in the hard setting) in a much more general framework
where the learner and adversary distributions are infinitely continuous.
After defining this general setting we extend theorem~\ref{thm:LDRS_optimal} of Section~\ref{sec:characterization} on
hard LDRS optimality. The resulting Theorem~\ref{thm:LDRS-optimality-cont} is obtained by assuming that
the adversary strategy space is sufficiently large, now satisfying a continuous extension of $\PropA$
called $\PropAcont$
($\PropC$ is not required in the continuous setting).

The main contribution of this section is a reduction of the SCC problem to two-class classification problem.
The two-class classification is facilitated by sampling points from a synthetically generated
``other class.'' This other class is generated so that it is uniform over its support, which is appropriately selected
around the observed support of $P$. Using this synthetic sample we obtain a binary training set on which
we can train
a soft binary classifier. The final $\delta$-valid SCC classifier is then identified by selecting a
threshold on the classifier output so as to maximize the type I error up to $\delta$.
The entire routine is simple, practical and
if the underlying  two-class soft classifier learning algorithm runs in $C(n)$ time complexity, our SCC
algorithm runs in
time $O(C(n) + n)$.
An alternative approach where a hard two-class classifier can be used is described by \citeA{Nisenson2010}.


We show that the SCC routine obtained using this approach is consistent in the sense that if the underlying classification device
is consistent then the resulting one-class classifier is asymptotically an LDRF, thus providing
an optimal SCC solution when the adversary strategy space satisfies $\PropAcont$.

\subsection{Definitions}
The SCC problem in the continuous setting is essentially the same as in the finite case
(see Section~\ref{sec:formulation}) but now
both the source distribution $P$ and the adversary distribution can be infinitely continuous distributions
over $\Real^d$. Let $\lambda$ be the Lebesgue measure on $\Real^d$.
We assume that $P$
is absolutely continuous with respect to $\lambda$ (in other words, if a Borel set $b$ has zero volume in $\Real^d$,
then $P(b) = 0$).
Denote by $p$  the density function of $P$ and let $\supp(p)$ be its support in $\Real^d$.

We define the function $\ind_b(x) \eqdef \ind(x \in b)$, where $\ind(\cdot)$ is the indicator function.
For a Borel set $b$, we define $l_p(b) \eqdef b \bigcup \{x : p(x) = 0\}$.

\begin{definition}[Minimum Volume Set]
\label{def:min-vol-set}
A set $b \subseteq \supp(p)$ is called a minimum volume set of measure $1-\delta$
if $P(b) = 1-\delta$ and for all $b'$ such that $P(b') = P(b) = 1-\delta$,
$\lambda(b) \leq \lambda(b')$.
\end{definition}

\begin{definition}[Low Density Set]
\label{def:low-density-set}
~
\begin{enumerate}
\item[(i)]
Let $b \subseteq \supp(p)$ be a minimum volume set of measure $1-\delta$.
Let $m$ be any set such that $P(m) = \delta$ and $b \bigcap m = \emptyset$. Then, we call $m$ a \emph{core low density set w.r.t.~$P$ and $\delta$},
\item[(ii)]
Denote by $\cldp$ the set of all
core low-density sets w.r.t.~$P$ and $\delta$.
\item[(iii)]
We call a set $s$ a \emph{low density set w.r.t.~$P$ and $\delta$} if there
exists an $m \in \cldp$ such that $s = l_p(m)$.
\end{enumerate}
\end{definition}

\subsection{LDRS optimality in the continuous setting}

\begin{definition}[Low-Density Rejection Strategy (LDRS) and Function (LDRF)]
We define
\begin{align*}
LDRS_\delta(P) \eqdef& \left\{r(\cdot):\; \exists m \in \cldp\;s.t.\;
r(\cdot) \equiv \ind_{l_P(m)}(\cdot) \right\}.
\end{align*}
Any function $r(\cdot) \in LDRS_\delta(P)$
is called a $\delta$-tight Low-Density Rejection Function (LDRF),
and the Low-Density Rejection Strategy is to choose any $\delta$-tight LDRF.
\end{definition}

\begin{definition}[$\PropAcont$]
\label{def:PropA-cont}

We say that two Borel sets $j,k$ satisfy condition $(*)$ if: (i) $j,k \subset \supp(p)$;
(ii) $j \cap k = \emptyset$;
(iii) $P(j) = P(k)$;
and (iv) $\lambda(j) \geq \lambda(k)$.

An adversary strategy space $\cQ$ has
$\PropAcont$ w.r.t $P$, if
for every pair $j,k$ satisfying $(*)$:
$\forall Q \in \cQ$
such that $Q(j) < Q(k)$, $\exists Q' \in \cQ$,
for which
\begin{enumerate}
\item
$Q'(j) + Q(j) \geq Q'(k) + Q(k)$;
\item
For all Borel sets $b$ for which $b \bigcap \left(j \bigcup k\right) = \emptyset$,
$Q'(b) = Q(b)$.
\end{enumerate}
\end{definition}

The proof of the following theorem can be found in the appendix.
\begin{theorem}
\label{thm:LDRS-optimality-cont}
When the learner is restricted to hard-decisions and $\cQ$ satisfies $\PropAcont$ w.r.t.~$P$,
then LDRS is optimal.
\end{theorem}

\subsection{SCC via Two-Class Classification}
\label{subsec:occ_via_2class}
We propose an SCC routine that relies on a \emph{soft} binary classifier induction.
We can use any two-class algorithm, which is  consistent in the sense that it minimizes a loss function $\phi(\cdot)$ that is
non-negative, differentiable, convex, strictly convex over $[-\infty, 0)$ and satisfies $\phi'(0) < 0$.
These conditions are similar but stronger than the conditions required by
\citeA{Bartelett_loss}, which provide necessary and sufficient conditions for a convex $\phi$ to be \emph{classification-calibrated}.\footnote{Our additional conditions are differentiability everywhere and
strict convexity over $[-\infty, 0)$ .
The reason for these extra conditions is that we threshold the soft classifier's output and don't merely use its sign for classification.}
We note however that the commonly used loss functions as discussed in \citeA{Bartelett_loss} satisfy our conditions, including
the quadratic, truncated-quadratic, exponential and logistic loss functions, to name a few.
In the extensions to this section \cite<see>{Nisenson2010}
an SCC routine is presented
that can utilize any \emph{hard} binary classifier induction algorithm that minimizes
either the 0/1, $L_1$, or  hinge loss functions, as well as any
of the loss functions defined by \citeA{Bartelett_loss}.\footnote{The use of a hard classifier (as opposed to a soft one) results in a time complexity penalty of a factor of
$O(\log n)$.}

Our SCC algorithm is given a training sample $S_n= \{x_1,\ldots,x_n\}$ of $n$ training examples drawn i.i.d. from an unknown
source distribution $P$ over $\Real^d$. Given a type-I threshold $\delta$ the algorithm outputs a hard rejection function $r(\cdot)$ over $\Real^d$.
The main idea of the algorithm is based on the following observation. If our domain is bounded, we can define
a two-class classification problem where the first class is $P$ and the other class is a uniform distribution over the
(bounded) domain. Then, the output of a consistent soft binary classifier is strictly monotonically increasing with
$p(\cdot)$ (the density of $P$) over the support of $P$ (it is only weakly monotone in $p(\cdot)$ over the whole domain).
Therefore, thresholding the classifier's output, with an appropriate quantile,
identifies a $\delta$-valid level-set in $P$, inducing a rejection function.

In practice, sampling from a uniform distribution over large domains is computationally hard and moreover, undefined
for unbounded domains. Our algorithm avoids these obstacles by sampling uniformly in grid cells containing
sampled points from $P$.
An additional complication arises in cases where the density
$p$ is flat over some regions, which results in discontinuities of the
level sets. This is a known issue in level set estimation and is often avoided by assuming that there are no
flat regions in $p$, in particular in regions corresponding to the $\delta$ level set \cite{Tsybakov97,Molchanov90}.
We don't assume this; our algorithm handles flat regions in $p$
by jittering the classifier output using a small and vanishing (in $n$)
random noise (see step 6 in the algorithm below).
The resulting algorithm is computationally efficient and practical.

A major component of our algorithm is determining a threshold by quantile estimation.
This occurs in Step~7 of the algorithm. We apply a known estimator \cite{Uhlmann63,Zielinski04}
that is unbiased and has certain optimal characteristics (see below).
This quantile estimator assumes that the cumulative distribution function (cdf), $F$, underlying the sample, is continuous,
and is defined over $\Real$ (i.e., $F$ is the cdf of a real random variable).
Let $t_\mu$ be
the estimate of the $\mu$-quantile of $F$, given $n$ sample points drawn i.i.d.~according to $F$. The estimator is \emph{unbiased} if $\rE_F[F(t_\mu)] = \mu$.
Its variance is $Var_F[F(t_\mu)]$. The estimator we use
is called the ``uniformly minimum variance unbiased estimator.''
It was introduced by \citeA{Uhlmann63} and we rely on analysis by~\citeA{Zielinski04}.
This estimator can only
be used for estimating $\mu$-quantiles that satisfy $\frac{1}{n+1} \leq \mu \leq \frac{n}{n+1}$, which is equivalent to requiring that
$n \geq \max\left\{\frac{\mu}{1-\mu}, \frac{1-\mu}{\mu} \right\}$. The estimator chooses an index $\pi_\mu$ in $[1, \dots, n]$,
and the estimate of the $\mu$-quantile is the $\pi_\mu$-th order statistic; in other words, if our sample points are sorted in increasing order, then
the estimate is the $\pi_\mu$-th element. $\pi_\mu$ is calculated as follows:
\begin{itemize}
  \item Set $k \eqdef \lfloor (n+1)\mu \rfloor$.
  \item Set $\beta \eqdef (n+1)\mu - k$.
  \item With probability $\beta$, set $\pi_\mu = k+1$, and with probability $1-\beta$, set $\pi_\mu = k$.
\end{itemize}
The estimator's variance is \cite{Zielinski04}:
$$
\frac{\beta(1-\beta)}{(n+1)(n+2)} + \frac{\mu(1-\mu)}{n+2}.
$$
The variance is maximized when $\beta = \mu = \frac{1}{2}$, and thus the variance is at most $\frac{1}{4(n+1)}$.
Moreover, according to \citeA{Zielinski04}, the estimator
is unbiased and its variance is not greater than that of any other unbiased estimator within the family of estimators
that can be defined using a probability distribution over single order statistics.
For very small samples with $n < \max\left\{\frac{\mu}{1-\mu}, \frac{1-\mu}{\mu} \right\}$, we ``fall-back'' to a simple ``default'' estimator, which sets $\pi_\mu \eqdef \lceil n\mu\rceil$.
We term this quantile-estimation algorithm the ``uniformly minimum variance unbiased (with fall-back) estimator,'' or the ``UMVUFB estimator.''

The algorithm is as follows:
\begin{enumerate}
\item  Define a grid over $\Real^d$  with arbitrary origin and positive cell side length $g_n$.
Let $g_n \to 0$, be such that $ng^d \to \infty$. For example, $g_n \eqdef n^{-\frac{1}{d+2}}$.
Select an arbitrary origin $x_0$, for example, uniformly at random from the unit-hypercube.
For any point $x = (x^{(1)},\ldots,x^{(d)})$, define
the function
$$
A_n(x)\eqdef \left\lfloor\frac{x - x_0}{g_n}\right\rfloor \eqdef
\left(
	\left\lfloor\frac{x^{(1)} - x_0^{(1)}}{g_n}\right\rfloor, \left\lfloor\frac{x^{(2)} - x_0^{(2)}}{g_n}\right\rfloor, \dots, \left\lfloor	\frac{x^{(d)} - x_0^{(d)}}{g_n}\right\rfloor
\right).
$$
For each point $x$, $A_n(x)$ specifies the coordinates of the ``lower left'' corner of the grid cell containing $x$.
\item Define the set $G^n_P = \bigcup_{x \in S_n} A_n(x)$ of covered grid cell corners.
\item Generate an artificial sample $O_n$ of size $n$ from the ``other class.''
Each point is selected independently at random as follows:
\begin{enumerate}
  \item Choose $a \in G^n_P$ uniformly at random.
  \item Choose a point $x$ uniformly at random from the unit-hypercube.
  \item The new artificial sample point is $o \eqdef a + g_n \cdot x$.
\end{enumerate}
\item Using the training sample consisting of $S_n$ (labeled $+1$) and $O_n$ (labeled $-1$),
train a soft binary classifier $h_n(\cdot)$.
\item
Define a confidence margin for the $\delta$ threshold.
Select any $\theta_n \to \infty$ such that $\theta_n  = o(\sqrt{n})$,
for example, take $\theta_n \eqdef \sqrt[3]{n}$.
Now define $\delta^+_n \eqdef \delta + \frac{1}{\theta_n}$.
Choose $\delta^-_n \leq \delta - \frac{1}{\theta_n}$ be such that $\delta^-_n \to \delta$.

\item
Jitter the classifier output.
Let $X_P$ be a random variable where $X_P \sim P$ and $Y_n \eqdef h_n(X_p)$.
Let $\Phi(\cdot)$ be the cumulative distribution function of $N(0,1)$, and
let $m_n$ be such that $\frac{\Phi(-m_n)}{\Phi(m_n)} = o\left(\frac{1}{\theta_n}\right)$, for example
$m_n \eqdef e^n$.
Let $\sigma_n \eqdef o\left(\frac{1}{m_n}\right)$,
for example, $\sigma_n \eqdef \frac{e^{-n}}{m_n} = e^{-2n}$.
Let $\varepsilon \sim N[0, \sigma^2_n]$, and set $Z_n \eqdef Y_n + \varepsilon$.

\item
We use the following threshold mechanism. We will select two thresholds $t^-_n$ and $t^+_n$ on $Z_n$.
The cutoff is always $t^-_n$ and it is inclusive when $t^-_n < t^+_n$.
Specifically, let $t^-_n$ and $t^+_n$ be estimates of the $\left(\Phi(m_n)\delta^-_n + \frac{\Phi(-m_n)}{2}\right)$-quantile and $\left(\Phi(m_n)\delta^+_n + \frac{\Phi(-m_n)}{2}\right)$-quantile
of $Z_n$, respectively. In order to establish these estimates we require a sample from $Z_n$. The following procedure produces
a list of sample points $S_Z$.
\begin{itemize}
  \item Set $S_Z = []$, i.e. $S_Z$ is an empty list.
  \item For each $x \in S_n$: Choose a value $\epsilon_x \sim N[0,\sigma_n^2]$ and append the value $h_n(x) + \epsilon_x$ onto $S_Z$.
\end{itemize}

The sample $S_z$ is then the input to the UMVUFB estimator defined above.
%

\item Define the rejection function
\begin{align*}
r_n(x) \eqdef \begin{cases}
	1 & A_n(x) \not\in G^n_P;\\
    \ind(h_n(x) \leq t^-_n) & A_n(x) \in G^n_P \textrm { and }\; t^-_n < t^+_n;\\
    \ind(h_n(x) < t^-_n) & \textrm{ otherwise. }
 \end{cases}
\end{align*}
\end{enumerate}

\begin{remark}
Instead of a soft classifier, $h_n(\cdot)$ could have been any consistent class-probability estimator, where $h_n(x)$ is the estimate of $\Pr\{+1|x\}$. See \cite{Nisenson2010} for details.
$h_n(\cdot)$ could also be a consistent ranking algorithm \cite<see, e.g.,>{ClemenconLugosiaVayatis05}.
In this case, the quantile estimator must select a single sample point to represent the quantile.
All comparison operations (e.g. $<$, $\leq$), including those done by the quantile estimator, must be performed by the ranking algorithm. The ranking algorithm must also be able to distinguish between $t^-_n < t^+_n$ and $t^-_n = t^+_n$.
\end{remark}

Let $U_n$ (with density $u_n$) be the distribution of $O_n$ (defined in Step~3).
Clearly, $u_n$ is uniform over its bounded support.
As previously noted, if the support of the generated distribution is significantly larger than that of $P$, an exorbitant number
of points may need to be generated in practice in order to reject low density areas in $P$ \cite{Davenport06learningminimum}.
The following lemma shows that the probability of generating points outside of $p$'s support, almost surely tends to zero.
\begin{lemma}
\label{lem: bayesQ-creation-supp-cont}
$U_n(\Real^d \setminus \supp(p)) \aslim 0$.
\end{lemma}
\begin{proof}
Recall that $g_n$ is a sequence of positive numbers such that $\lim_{n \to \infty} ng_n^d = \infty$ and $\lim_{n \to \infty} g_n = 0$.
Define a sequence of positive numbers $g'_n$, such that $g'_n \geq 2g_n$, $g'_n \to 0$ and $\lim_{n \to \infty} \frac{n{g'}_n^d}{\log n} = \infty$.
Define $A(x, g'_n) \eqdef \{y: y \in \Real^d \textrm{ and } ||x-y||_\infty \leq g'_n\}$.
Define $T_n \eqdef \bigcup_{i=1}^n A(x_i, g'_n)$. \citeA{DevroyeWise80} show that
for any probability measure $\nu$ on the Borel sets of $\Real^d$ whose restriction to $\supp(p)$
is absolutely continuous w.r.t.~$P$, it holds that
$\nu(T_n \symdif \supp(p)) \aslim 0$.
We note that the grid cell of $x$ is always a sub-region of $A(x, g'_n)$,
and therefore $\supp(u_n) \subseteq T_n$. Thus, noting that $U_n(\Real^d \setminus \supp(p)) = U_n(\supp(u_n) \setminus \supp(p))$,

\begin{align*}
U_n(\supp(u_n) \setminus \supp(p)) &\leq \sup_{m} U_m(\supp(u_n) \setminus \supp(p)) \\
&\leq  \sup_{m} U_m(T_n \setminus \supp(p))\\
&\leq  \sup_{m} U_m(T_n \symdif \supp(p))\\
\Rightarrow \lim_{n \rightarrow \infty} U_n(\supp(u_n) \setminus \supp(p)) &\leq \lim_{n \rightarrow \infty}
\sup_{m} U_m(T_n \symdif \supp(p))\\
&= \sup_{m} \lim_{n \to \infty} U_m(T_n \symdif \supp(p))\\
&\aseq 0.
\end{align*}
\end{proof}
\begin{remark}
It is difficult to establish exact convergence rates in Lemma~\ref{lem: bayesQ-creation-supp-cont}
without constraints on $P$.
For cases where $\lambda(\Real^d \setminus \supp(p)) = 0$, we obviously have that $U_n(\Real^d \setminus \supp(p)) = 0$.
This is the case, for example, for finite mixtures of Gaussians.

If there exists a constant $K$, such that $|p(x) - p(y)| \leq K$
whenever $||x-y||_\infty \leq g_n$, we can establish an upper bound on the rate.
The condition $||x-y||_\infty \leq g_n$ is equivalent to $x$ and $y$ being in the same grid-cell.
Therefore, if $p(x) > K$, then for all $y$ in the same grid-cell,
$p(y) > 0$. Note that if $p$ is Lipschitz continuous such that $|p(x) - p(y)| \leq \frac{K}{g_n}||x-y||_\infty$,
then $p$ meets the above condition.
Let $1-\eta$ be a desired confidence level.
Let $C^\eta_K$ be the number of cells in the grid
which contain more than $Kng_n^d + \sqrt{\frac{n(\ln |G^n_p| -\ln\eta)}{2}}$ sample points.
Then, using Hoeffding's inequality, it isn't hard to show
that with probability at least $1-\eta$, $U_n(\Real^d \setminus \supp(p)) \leq 1 - \frac{C^\eta_K}{|G^n_P|}$.
\end{remark}

\begin{definition}[Quantile]
Let $\xi$ be a random variable whose domain is in $\Real$.
We say that $t$ is a \emph{$\mu$-quantile} of $\xi$ if
$$
t \in S_\xi(\mu) \eqdef
\left\{ \tau \in \Real: \Pr\{\xi < \tau\} \leq \mu \textrm{ and } \Pr\{\xi \leq \tau\} \geq \mu\right\}
.$$
\end{definition}

Define a new random variable $\kappa$ to represent the level sets of $P$. Formally, its cumulative distribution function is
$F_\kappa(t) = P(\{x: p(x) \leq t\})$.

\begin{definition}
Let $v$ be any $\delta$-quantile of $\kappa$.
We say $P$ has a \emph{$\delta$-jump} if $F_\kappa(v) > \delta$.
\end{definition}

\begin{figure}[ht]
\label{fig:noDeltaJumps}
\centerline{
\begin{tabular}{cc}
\psfig{figure=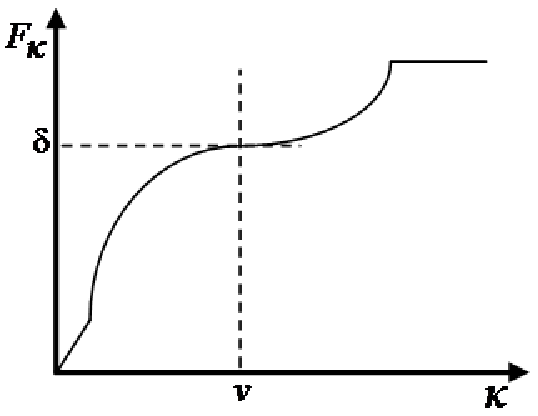,height=5.5cm}
&
\psfig{figure=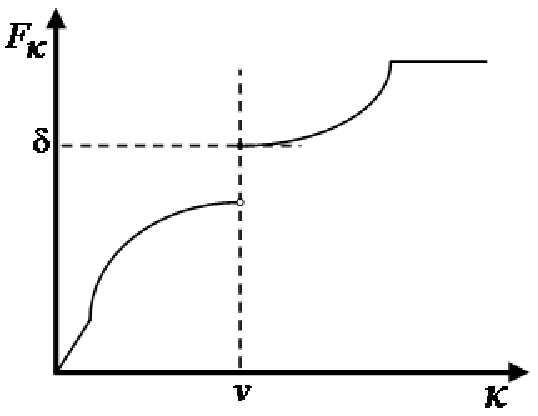,height=5.5cm}\\
(a) & (b)
\end{tabular}
}
\caption{
The cumulative distribution function $F_\kappa$ when $P$ doesn't have a $\delta$-jump.
}
\end{figure}

\begin{figure}[ht]
\label{fig:deltajumps}
\centerline{
\begin{tabular}{cc}
\psfig{figure=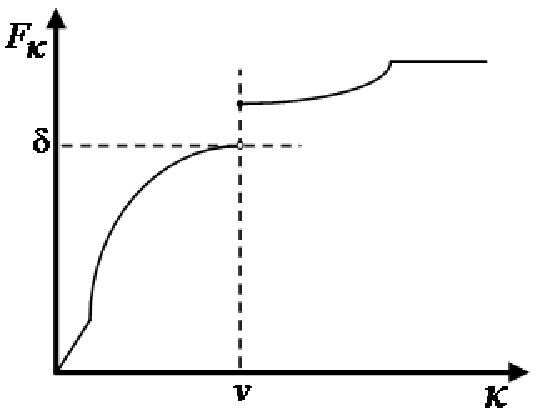,height=5.5cm}
&
\psfig{figure=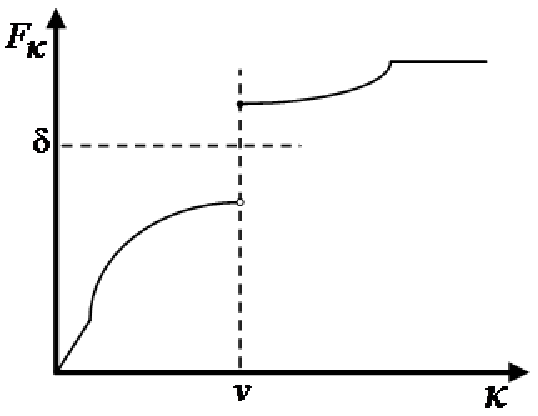,height=5.5cm}\\
(a) & (b)
\end{tabular}
}
\caption{
The cumulative distribution function $F_\kappa$ when $P$ does have a $\delta$-jump.
}
\end{figure}

We now will consider two cases,
one where $P$ doesn't have a $\delta$-jump and one where it does. See Figure~2 
and Figure~3. 
In all figures
the (unique) $\delta$-quantile of $\kappa$ is marked by $v$. Note that a quantile need not be unique, in particular there
will be a range of values wherever $F_\kappa$ is flat.

\begin{definition}
A rejection function $r(\cdot)$ is called a \emph{$\delta$-maximal level-set estimator for $P$} if,
for some $v \in S_\kappa(\delta)$, either:
\begin{enumerate}
  \item $P$ doesn't have a $\delta$-jump and $r(x) \equiv \ind(p(x) \leq v)$, almost everywhere.
  \item $P$ has a $\delta$-jump and $r(x) \equiv \ind(p(x) < v)$, almost everywhere.
\end{enumerate}
\end{definition}

Note that if $P$ doesn't have a $\delta$-jump, then a $\delta$-maximal level-set estimator for $P$ is a $\delta$-tight LDRF.
We will now prove that the output of the algorithm is asymptotically (almost surely)
a $\delta$-maximal level-set estimator for $P$.

\begin{theorem}
\label{thm:uniform-to-ldrs}
Let $\{U'_n\}$, $n=1,2,\ldots,$ be a sequence of probability measures such that for each $n$, $U'_n$ has uniform density $u'_n$ over
its bounded support, and $\lim_{n \to \infty} P(\supp(u'_n)) = 1$.
Define a Bayesian binary classification problem for each $n$.
Let the first class, $c_1 \equiv +1$ have distribution $P$,
and the second class $c_2 \equiv -1$ have distribution $U'_n$.
The classes' prior probabilities
are $\Pr\{+1\} = \Pr\{-1\} = \frac{1}{2}$.
Let $\phi(\cdot)$ be a non-negative, differentiable, convex loss function such that it is strictly convex on $[-\infty, 0)$ and $\phi'(0) < 0$.
Let $h^*_n(\cdot)$ be the soft Bayes-optimal classifier that minimizes the expected loss.
Define a random variable $Y^*_n \eqdef h^*_n(X_P)$.
Let $t^*_n$ be a $\delta$-quantile of $Y^*_n$.
Define the rejection function:
\begin{align*}
r^*_n(x) \eqdef \begin{cases}
	1 & x \not\in \supp(u'_n);\\
    \ind(h^*_n(x) \leq t^*_n) & x \in \supp(u'_n) \textrm{ and $P$ doesn't have a $\delta$-jump}; \\
    \ind(h^*_n(x) < t^*_n) & \textrm{ otherwise. }
 \end{cases}
\end{align*}
Then, $r^*(\cdot) \eqdef \lim_{n \to \infty} r^*_n(\cdot)$ is a $\delta$-maximal level-set estimator for $P$.
\end{theorem}
\begin{proof}
We first consider $x \in \supp(u'_n)$.
Define the function $\psi_n(x) \eqdef \frac{p(x)}{u'_n(x)}$, defined over $\supp(u'_n)$.
From Bayes theorem, it is not hard to show that $\Pr\{+1|x\} = \frac{p(x)}{p(x) + u'_n(x)}$.
The loss for a point $x$ when we assign it value $y$ is \cite{Bartelett_loss}:
\begin{align*}
\ell(x,y) \eqdef& \Pr\{+1|x\}\phi(y) + \Pr\{-1|x\}\phi(-y)\\
 &= \frac{p(x) \phi(y) + u'_n(x)\phi(-y)}{p(x) + u'_n(x)}.
\end{align*}
It is easy to verify that
for a fixed $x$, at the minimum (over $y$) of $\ell(x,y)$,
$p(x)\phi'(y) = u'_n(x)\phi'(-y)$.
Alternatively: $\phi'(-y) = \psi_n(x)\phi'(y)$.
Let $x_1$ and $x_2$ be two points such that $\psi_n(x_1) > \psi_n(x_2)$.
Note that $\min\{\phi'(y), \phi'(-y)\} \leq \phi'(0) < 0$ for all $y$.
Let $c_i \eqdef \psi_n(x_i)$
and $y_i$ be a solution to $\phi'(-y) = c_i \phi'(y)$, for $i \in \{1,2\}$.
Note that $c_1,c_2 \geq 0$ and therefore, in order for equality to occur
it is necessary that $\phi'(y_i), \phi'(-y_i) \leq 0$ (with equality only if $c_i = 0$).
We can now rewrite  $\phi'(-y_i) = c_i \phi'(y_i)$ as $|\phi'(-y_i)| = c_i |\phi'(y_i)|$.

We will now prove that $y_1 > y_2$. Assume by contradiction that the statement is
false. Then $y_2 \geq y_1$. Therefore, $|\phi'(y_2)| \leq |\phi'(y_1)|$ and $|\phi'(-y_2)| \geq |\phi'(-y_1)|$.
Since $\psi_n(x_1) > \psi_n(x_2)$, it follows that $c_1 > c_2$.
If $c_2 = 0$, then $0 = |\phi'(-y_2)| \geq |\phi'(-y_1)|$. Therefore $\phi'(-y_1) = 0$
and $\phi'(y_1) < 0$, which gives $0 = |\phi'(-y_1)| = c_1 |\phi'(y_1)| < 0$, which is a contradiction. Thus, $c_2 \neq 0$, and
$|\phi'(-y_2)| = c_2|\phi'(y_2)| \leq c_2|\phi'(y_1)| = \frac{c_2}{c_1} |\phi'(-y_1)|
\leq \frac{c_2}{c_1} |\phi'(-y_2)| < |\phi'(-y_2)|$. Contradiction.

Now consider the case where $c_i = \frac{|\phi'(-y_i)|}{|\phi'(y_i)|} > 0$. Therefore, $\phi'(y_i), \phi'(-y_i) < 0$. Since $\phi(\cdot)$
is strictly convex over $[-\infty, 0)$ it follows that as $y_i$ increases
$|\phi'(y_i)|$ decreases and $|\phi'(-y_i)|$ increases. Therefore, if $\psi_n(x_1) = \psi_n(x_2) > 0$, there
is a unique solution.

Therefore, $h^*_n(x)$ is monotonically increasing with $\psi_n(x)$, almost everywhere over $\supp(u_n)$
and strictly monotonically increasing with $\psi_n(x)$, almost everywhere over $\supp(u_n) \bigcap \supp(p)$.
Since $u'_n(\cdot)$ is constant over its support, this implies: $p(x_1) < p(x_2) \Rightarrow h^*_n(x_1) < h^*_n(x_2)$, and
$0 < p(x_1) = p(x_2) \Rightarrow h^*_n(x_1) = h^*_n(x_2)$.
Therefore, for some $v^*_n$, $\{x \in \supp(u'_n) \bigcap \supp(p):  h^*_n(x) \leq t^*_n\}$ is identical to $\{x \in \supp(u'_n) \bigcap \supp(p): p(x) \leq v^*_n\}$
(with the possible exception of a set of points of zero Lebesgue measure).
Recalling that $Y^*_n = h^*_n(X_P)$ for $X_P \sim P$ and that $P(\supp(u'_n)) \to 1$:
\begin{align*}
\lim_{n \to \infty} t^*_n \in&  \lim_{n \to \infty} \{\tau \in \Real: \Pr\{Y^*_n < \tau\} \leq \delta \textrm{ and } \Pr\{Y^*_n \leq \tau\} \geq \delta\}\\
 =&  \lim_{n \to \infty} \{\tau \in \Real: P(\{x: h^*_n(x) < \tau\}) \leq \delta \textrm{ and }  P(\{x: h^*_n(x) \leq \tau\}) \geq \delta\}\\
 =&  \lim_{n \to \infty} \{\tau \in \Real: P(\{x \in \supp(u'_n) \bigcap \supp(p): h^*_n(x) < \tau\}) \leq \delta \textrm{ and }\\
 \;&\;\;\;\;\;\;\;\;\;\;\;\;\;\;\;\;\;P(\{x \in \supp(u'_n) \bigcap \supp(p): h^*_n(x) \leq \tau\}) \geq \delta\}\\
\Rightarrow
\lim_{n \to \infty} v^*_n \in&
\lim_{n \to \infty} \{\tau' \in \Real: P(\{x \in \supp(u'_n) \bigcap \supp(p): p(x) < \tau'\}) \leq \delta \textrm{ and }\\
 \;&\;\;\;\;\;\;\;\;\;\;\;\;\;\;\;\;\;P(\{x \in \supp(u'_n) \bigcap \supp(p): p(x) \leq \tau'\}) \geq \delta\}\\
  =&  \lim_{n \to \infty} \{\tau' \in \Real: P(\{x: p(x) < \tau\}) \leq \delta \textrm{ and }  P(\{x: p(x) \leq \tau'\}) \geq \delta\}\\
  =& \{\tau' \in \Real: P(\{x: p(x) < \tau\}) \leq \delta \textrm{ and }  P(\{x: p(x) \leq \tau'\}) \geq \delta\}\\
  =& \{\tau' \in \Real: \Pr\{\kappa < \tau'\} \leq \delta \textrm{ and }  \Pr\{\kappa \leq \tau'\} \geq \delta\}\\
  =& S_\kappa(\delta)
\end{align*}
Therefore, let $v^\delta_p \in S_\kappa(\delta)$ be such that
$v^\delta_p = \lim_{n \to \infty} v^*_n$.
Note that since $\delta > 0$, $v^\delta_p > 0$
(otherwise $\delta \leq P(\{x: p(x) \leq v^\delta_p) = 0$).
Therefore, for sufficiently large $n$, $v^*_n > 0$.

Let us assume that $P$ doesn't have a $\delta$-jump.
Therefore, for almost every $x \in \supp(u'_n)$, $\ind(h^*_n(x) \leq t^*_n)= \ind(p(x) \leq v^*_n)$. Then almost
everywhere in $\supp(u'_n)$: $r(x) = \lim_{n \to \infty} \ind(h^*_n(x) \leq t^*_n) = \ind(p(x) \leq v^\delta_p)$.
It is given that $P(\Real^d \setminus \supp(u'_n)) \to 0$. Therefore, $\lambda(\{x \not\in \supp(u'_n): p(x) > v^\delta_p\}) \to 0$.
which is equivalent to $\lambda(\{x \not\in \supp(u'_n): \ind(p(x) \leq v^\delta_p) \neq r^*_n(x)\}) \to 0$.

If $P$ has a $\delta$-jump, the proof is almost identical, only with minor changes in the strengths of inequalities. For almost every $x \in \supp(u'_n)$:
$r(x) = \lim_{n \to \infty} \ind(h^*_n(x) < t^*_n) = \ind(p(x) < v^\delta_p)$, and $\lambda(\{x \not\in \supp(u'_n): \ind(p(x) < v^\delta_p) \neq r^*_n(x)\}) \to 0$.
\end{proof}

We will now make clear the relation between the algorithm given and Theorem~\ref{thm:uniform-to-ldrs}.
Clearly $\{U_n\}$
is a series of distributions each having a uniform density, $u_n$, over its bounded support.
We will now prove that $P(\supp(u_n)) \aslim 1$.

\begin{lemma}
For any $\epsilon > 0$, $\Pr\{P(\supp(u_n)) \leq 1 - \epsilon\} \leq 2e^{-2n\epsilon^2 - n o(1)}$ and $P(\supp(u_n)) \aslim 1$.
\end{lemma}
\begin{proof}
We define $G(x)$ to be the cell in the grid which contains $x$.
Define $c(b) \eqdef \left|\{S \bigcap b\}\right|$ to be the count of the number
of training samples which fall within set $b$.
Then the histogram density estimate is $\tilde{p}_n(x) \eqdef \frac{c\left(G(x)\right)}{nh^d}$.
As shown by \citeA{DevroyeGyorfi02} in Theorem~5.6,
$\Pr\left\{\int_{\Real^d} |p(x) - \tilde{p}_n(x)|\lambda(dx) > 2\epsilon\right\} \leq 2e^{-2n\epsilon^2 - n o(1)}$.
However, since $P$ is absolutely continuous w.r.t.~$\lambda$, it
follows from Scheff\'{e}'s theorem (\citeA{Scheffe47}, used as Theorem~5.4 by \citeA{DevroyeGyorfi02}), that for any Borel set $B$ over $\Real^d$,
$\Pr\left\{\int_{B} |p(x) - \tilde{p}_n(x)|\lambda(dx) > \epsilon\right\} \leq 2e^{-2n\epsilon^2 - n o(1)}$.

By definition, $\tilde{p}_n(x) = 0$ for all
$x \not\in \supp(u_n)$. Therefore:
\begin{align*}
\Pr\left\{P\left(\Real^d \setminus \supp(u_n)\right) > \epsilon\right\} =& \Pr\left\{\int_{\Real^d \setminus \supp(u_n)} |p(x) - \tilde{p}_n(x)|\lambda(dx) > \epsilon\right\}\\
 \leq& 2e^{-2n\epsilon^2 - n o(1)}.
\end{align*}
Since this is true for any $\epsilon$, it immediately follows that $\Pr\{\lim_{n \to \infty} P(\Real^d \setminus \supp(u_n)) \neq 0\} = 0$, or $P(\supp(u_n)) \aslim 1$.
\end{proof}

Therefore, the only remaining part is to show how $t^-_n$ and $t^+_n$ relate to $t^*_n$ and to whether $P$ has a $\delta$-jump or not.
We note that for all $x$, at the limit, $h^*_n(x) = h_n(x)$
and therefore, $Y^*_n \equiv Y_n$ (i.e. they are distributed identically).
For sufficiently large $n$,
the quantile estimator used is unbiased with standard deviation vanishing at a rate of $O\left(\frac{1}{\sqrt{n+1}}\right)$ \cite{Zielinski04}.
Therefore, since $\theta_n = o(\sqrt{n})$, it follows that $\frac{1}{\sqrt{n}} = o(\frac{1}{\theta_n})$, and thus
for sufficiently large $n$, $t^-_n$ and $t^+_n$ are tightly concentrated around a $\left(\Phi(m_n)\delta^-_n + \frac{\Phi(-m_n)}{2}\right)$-quantile (which is not greater than the $\left(\Phi(m_n)\left[\delta - \frac{1}{\theta_n}\right] + \frac{\Phi(-m_n)}{2}\right)$-quantile)
and a $\left(\Phi(m_n)\left[\delta + \frac{1}{\theta_n}\right] + \frac{\Phi(-m_n)}{2} \right)$-quantile for $Z_n$, respectively.
Therefore, since $\frac{\Phi(-m_n)}{\Phi(m_n)} = o\left(\frac{1}{\theta_n}\right)$, by the following lemma, $t^-_n$ and $t^+_n$ are also tightly concentrated around a $\delta^-_n$-quantile and a $\delta^+_n$-quantile for $Y_n$, respectively.

\begin{lemma}
\label{lem: YZ-concentration}
Let $m \geq 0$.
Let $t_\mu$ be a $\left(\Phi(m)\mu + \frac{\Phi(-m)}{2}\right)$-quantile of $Z_n$.
Then, for some $\mu'$ such that $|\mu' - \mu| \leq \frac{\Phi(-m)}{2\Phi(m)}$, a $\mu'$-quantile of $Y_n$, $t^*_{\mu'}$, satisfies
$|t^*_{\mu'} - t_\mu| \leq m\sigma_n$.
\end{lemma}
\begin{proof}
\begin{align*}
\mu + \frac{\Phi(-m)}{2\Phi(m)} \geq& \frac{1}{\Phi(m)}\Pr\{Z_n < t\} = \frac{1}{\Phi(m)}\Pr\{Y_n < t + \varepsilon\}\\
 \geq& \frac{1}{\Phi(m)}\Pr\{Y_n < t - m\sigma_n\}\Pr\{\varepsilon \geq -m\sigma_n\} = \Pr\{Y_n < t - m\sigma_n\}
\end{align*}
\begin{align*}
\mu + \frac{\Phi(-m)}{2\Phi(m)} \leq& \frac{1}{\Phi(m)}\Pr\{Z_n \leq t\} = \frac{1}{\Phi(m)}\Pr\{Y_n \leq t + \varepsilon\}\\
 \leq& \frac{1}{\Phi(m)}\left[\Pr\{Y_n \leq t + m\sigma_n\}\Pr\{\varepsilon \leq m\sigma_n\}
 + \Pr\{\varepsilon > m\sigma_n\}\right]\\
 \leq& \frac{1}{\Phi(m)}\left[\Pr\{Y_n \leq t + m\sigma_n\}\Phi(m) + \Phi(-m)\right]\\
 =& \Pr\{Y_n \leq t + m\sigma_n\} + \frac{\Phi(-m)}{\Phi(m)}.
\end{align*}
Therefore, $\mu + \frac{\Phi(-m)}{2\Phi(m)} \geq \Pr\{Y_n < t - m\sigma_n\}$ and $\mu - \frac{\Phi(-m)}{2\Phi(m)} \leq \Pr\{Y_n \leq t + m\sigma_n\}$.
Let $\Delta_1 \eqdef \Pr\{Y_n < t - m\sigma_n\} - \mu$, and
let $\Delta_2 \eqdef \Pr\{Y_n \leq t + m\sigma_n\} - \mu$.
Note that $t-m\sigma_n$ is a $(\mu + \Delta_1)$-quantile and that $t+m\sigma_n$ is a $(\mu + \Delta_2)$-quantile of $Y_n$.
Therefore, since $\Delta_2 \geq \Delta_1$,
for every $\Delta \in [\Delta_1, \Delta_2]$, there is some
$t'$ such that $|t' - t_\mu| \leq m\sigma_n$,
which is a $(\mu + \Delta)$-quantile of $Y_n$.
To complete the proof, note that $\Delta_1 \leq \frac{\Phi(-m)}{2\Phi(m)}$ and $\Delta_2 \geq -\frac{\Phi(-m)}{2\Phi(m)}$.
Therefore, there exists some $\Delta \in \left[-\frac{\Phi(-m)}{2\Phi(m)}, \frac{\Phi(-m)}{2\Phi(m)}\right]$ such
that $\Delta \in [\Delta_1, \Delta_2]$.
\end{proof}

\begin{theorem}
The rejection function output by the algorithm is (almost surely) identical to that of Theorem~\ref{thm:uniform-to-ldrs} at the limit, where $U'_n \eqdef U_n$.
\end{theorem}
\begin{proof}
By definition, $x \in \supp(u'_n) \Leftrightarrow A_n(x) \in G^n_p$.

We represent by $t^{*-}_n$ and $t^{*+}_n$ the $\delta^-_n$ and $\delta^+_n$ quantiles of $Y^*_n$, around which (for sufficiently large $n$),
$t^-_n$ and $t^+_n$ are tightly concentrated. In particular, $t^{*-}_n < t^{*+}_n \Rightarrow t^-_n < t^+_n$.
Note that $t^-_n \leq t^+_n$ and $t^{*-}_n \leq t^*_n \leq t^{*+}_n$ always,
and at the limit, $t^-_n = t^{*-}_n = t^*_n = t^{*+}_n = t^+_n$.
We now consider four cases.

In the first, $P$ doesn't have a  $\delta^-_n$-jump or a $\delta$-jump (see Figure~7.1(a)).
Then, $t^{*-}_n < t^*_n < t^{*+}_n$.
Therefore, for $x \in \supp(u'_n)$, at the limit: $\ind(h_n(x) \leq t^-_n) = \ind(h^*_n(x) \leq t^*_n(x))$.

In the second, $P$ has a $\delta^-_n$-jump but it doesn't have a $\delta$-jump (see Figure~7.1(b)).
Then, for sufficiently large $n$, $t^{*-}_n = t^*_n < t^{*+}_n$ and $\delta^-_n < \Pr\{Y^*_n \leq t^{*-}_n\} \leq \delta$.
Therefore, for $x \in \supp(u'_n)$, at the limit: $\ind(h_n(x) \leq t^-_n) = \ind(h^*_n(x) \leq t^*_n(x))$.

In the third, $P$ doesn't have a $\delta^-_n$-jump but it does have a $\delta$-jump (see Figure~7.2(a)).
Then, for sufficiently large $n$, $t^{*-}_n < t^*_n = t^{*+}_n$.
Therefore, for $x \in \supp(u'_n)$, at the limit: $\ind(h_n(x) \leq t^-_n) = \ind(h^*_n(x) < t^*_n(x))$.

In the fourth, $P$ has both a $\delta^-_n$-jump and a $\delta$-jump (see Figure~7.2(b)).
Then, for sufficiently large $n$, $t^{*-}_n = t^*_n = t^{*+}_n$.
Therefore, for $x \in \supp(u'_n)$, at the limit: $\ind(h_n(x) < t^-_n) = \ind(h^*_n(x) < t^*_n(x))$.
\end{proof}

\begin{remark}[Rates of Convergence and Finite Sample Notes]
The time complexity for our algorithm is $O(C(n) + n)$, where $C(n)$ is
the time complexity for the soft-classification algorithm.
The rate of convergence for the given algorithm is $\Theta\left(\frac{1}{\theta_n}\right) = \Theta\left(\frac{1}{n^{\frac{1}{2} + \epsilon}}\right)$, for
any $\epsilon > 0$, in addition
to the classifier's rate of convergence.\footnote{
The classifier doesn't truly need to minimize the loss.
Depending on the quantile-estimator, it is possible that only classifier errors which result in ``ordering
violations'' across the $\delta$-quantile can affect the output (beyond whether a strong or weak inequality is
used for testing the threshold). Thus, faster rates than the classifier's convergence rate to the minimum
may be possible. Also, ranking algorithms
\cite<see, e.g.,>{ClemenconLugosiaVayatis05} 
 could be used instead of soft-classification. In this case, achievable error
rates could provide (loose) upper bounds on such ordering violations.}
$\theta_n$ is only affected by the quantile-estimator used. In our case, the quantile
estimator utilized only requires that $F$, the cdf whose quantile is being estimated,
be continuous.
To meet this condition we added the noise term $\varepsilon$.
Note that $\varepsilon$ has no effect on the convergence rate; this is because
$\sigma_n$ can vanish as fast as desired. Similarly, by Lemma~\ref{lem: YZ-concentration}, we
can achieve arbitrarily tight bounds on the nearness of the quantiles of $Z_n$ and $Y_n$ by increasing the rate at which $m_n$ tends to infinity.

For finite sample sizes, some additional modifications are advisable. First, in order to ensure that $h^*_n(\cdot) = \rE_{U_n}[h_n(\cdot)]$, $O_n$ should be of size $n' \sim NB(n,\frac{1}{2})$,
and not $n$.
It is also possible to use a non-uniform prior probability (without affecting the algorithm's correctness), if
it is desired.
A validation set could be used for determining the quantile-estimates, rather than the training set. Note that for finite samples, it is not guaranteed
that $\Pr\{Y_n \leq t^-_n\} \approx \delta^-_n$. In fact, it is possible to be significantly larger if $Y_n$ has a large jump in
the range $[t^-_n - m\sigma_n, t^-_n)$.
Since by Lemma~\ref{lem: YZ-concentration}, $\Pr\{Y_n \leq t^-_n - m\sigma_n\} \leq \delta^-_n + \frac{\Phi(-m)}{2\Phi(m)} + o\left(\frac{1}{\theta_n}\right)$,
almost surely,
we can address this
issue by refining the definition of the rejection function output by the algorithm:
\begin{align*}
r_n(x) \eqdef \begin{cases}
	1 & A_n(x) \not\in G^n_P;\\
    \ind(h_n(x) \leq t^-_n - m\sigma_n) & A_n(x) \in G^n_P \textrm { and }\; t^-_n < t^+_n;\\
    \ind(h_n(x) < t^-_n - m\sigma_n) & \textrm{ otherwise. }
 \end{cases}
\end{align*}
Note that this fix isn't possible when using a ranking algorithm in place of a soft binary classifier, since only points, and not values, can be compared (i.e., $x$
and the chosen quantile point in $S_Z$ are compared in order to determine whether to reject $x$).

Finally, one needs to determine $\delta^-_n$, so that it is guaranteed (with high probability) that $\rho(r_n, P) \leq \delta$.
To accomplish this, one must take into account the quantile estimator used, since $\delta^-_n \leq \delta - \frac{1}{\theta_n}$, and $P(\Real^d \setminus \supp(u_n))$,
since this is always rejected.
It is known~\cite{HallHannan88} for the histogram density estimator, upon which the sampling of $O_n$ in the algorithm is loosely-based, that
$g_n$ of order $n^{-\frac{1}{d+2}}$ is optimal for minimizing $L_b$ distance for $1 \leq b < \infty$,
and that $g_n$ of order $\left(\frac{\log n}{n}\right)^\frac{1}{d+2}$ is the correct order
for minimizing $L_\infty$ distance. However, we are only interested in $P(\Real^d \setminus \supp(u_n))$.
We note that this is just the missing mass. Let $n_1$ be the number of grid cells
containing exactly one point from the sample. Then, as shown by \citeA{McallesterSchapire2000}, 
with probability at least $1-\eta$, $P(\Real^d \setminus \supp(u_n)) \leq \frac{n_1}{n} + \left(2\sqrt{2} + \sqrt{3}\right)\sqrt{\frac{\ln{3} - \ln{\eta}}{n}}$.
Clearly, increasing $g_n$ results in $n_1$ decreasing.
Therefore, $g_n$ should be large in order to minimize $P(\Real^d \setminus \supp(u_n))$ and small in order to minimize $U_n(\Real^d \setminus \supp(p))$ (since
if $g_n$ vanishes faster, $\lambda(\supp(u_n) \setminus \supp(p))$ decreases faster as well). This results
in a simple heuristic, namely to set $g_n$ to the smallest value such that $n_1 \leq t$, for some threshold $t$. For example, if we know the sample is ``clean'' in the sense that all points are drawn i.i.d. according to $P$, then we can take $t=0$. A larger value of $t$ could be chosen were we to suspect that the sample may contain noise, for example $t=\log{n}$. In general, it remains an open question of how $g_n$ should be optimized to balance between $P(\Real^d \setminus \supp(u_n))$ and $U_n(\Real^d \setminus \supp(p))$.
\end{remark}
\begin{remark}
\citeA{Cuevas97} use a plug-in approach to support estimation that can be leveraged here
to further decrease $U_n(\Real^d \setminus \supp(p))$ when $p$
has compact support and is continuously differentiable.
Let $g_n = cn^{-\frac{1}{d+2}}$ for some constant $c$, and let $\alpha_n$ be such that $\alpha_n^{-1} = o\left(g_n^{-1}\right)$. For example, $\alpha_n \eqdef \sqrt{g_n}$, or if $d = o\left(\frac{\log{n}}{\log\log{n}}\right)$, then $\alpha_n \eqdef \frac{1}{\log{n}}$.
Then, let $G^n_P$ only contain the ``lower-left'' corners of grid cells containing more than $n\alpha_n$
sample points. Since this only decreases $U_n(\Real^d \setminus \supp(p))$, Lemma~\ref{lem: bayesQ-creation-supp-cont}
remains correct and $U_n(\Real^d \setminus \supp(p)) \aslim 0$. Furthermore, $\lambda(\supp(p) \symdif \supp(u_n)) \aslim 0$~\cite{Cuevas97}, and thus $P(\Real^d \setminus \supp(u_n)) \aslim 0$, as well.
One may use results given by~\citeA{McallesterSchapire2000} to obtain an upper bound on $P(\Real^d \setminus \supp(u_n))$ for finite sample sizes.
\end{remark}
\begin{remark}
It may be possible to improve on the convergence rate for the quantile-estimator, by using more information than 1 to 2 order statistics.
This carries with it the risk of being less robust to classifier error. One such method is
kernel-based quantile regression \cite{ChristmannSteinwart08},
which is provably consistent. More complex quantile
estimation methods may be useful in improving the convergence rate, without
affecting the overall time complexity (dependent on the soft-classifier's time complexity), but these may exclude
the use of ranking algorithms, as the quantile estimation method may rely on more than the relative ordering of the sample points.
\end{remark}

\subsection{Discussion}
We have provided a computationally simple and consistent procedure for determining a $\delta$-maximal level set estimator, which for measures that don't have a $\delta$-jump, is also a $\delta$-tight LDRF.
While we have generated a uniform distribution for identifying low-density areas of $P$,
this is not strictly necessary. Indeed, to return to the investment analogy,
it is only necessary that the low-density areas have greater $ROI$ than the high density areas.
We term distributions which meet this condition as \emph{lenient adversarial distributions}.
The soft-classification approach used in this section applies for any such lenient adversarial distribution.
Indeed, lenient adversarial distributions can also be used when the underlying mechanism is a hard-classifier.
See \cite{Nisenson2010} for a full discussion on lenient adversarial distributions and their relation
to the existing SCC literature.
The importance of these results, including the generation of a ``tight'' lenient adversarial
distribution as given by the algorithm,
lies not only
in their justification of various approaches in the literature, but also in their applicability. Their only
requirements are on the loss function used and that $P$ be absolutely continuous with respect to the Lebesgue measure.
Since most common loss functions satisfy the requirements and the condition on $P$ is quite weak, a large body of results for regression and two-class classification can be utilized.

\section{On The Dual SCC Problem}
\label{sec:dual}

In the dual SCC problem the learner would like to guarantee the type II error, and minimize the type I error.
This problem can be relevant to intrusion detection and authentication applications as well as to data mining and novelty  detection. For example, in a biometric passport authentication system the authorities may mandate
a maximal intruder pass rate. Under this constraint one would clearly want to minimize the
false alarm rate. An alternative example is spam detection. A user may already have a two-class classification spam detection
system in place. This system may perform very well at detecting spam which is similar to previously encountered spam. However, spammers
are continually updating their spam so it will evade these filters. A second level system could be created, where
an SCC classifier is trained on the legitimate e-mails. Any e-mails which the first-level determines as legitimate would then
be tested against the second-level SCC classifier, which would either accept or reject them. A user may be willing to tolerate
a certain level of spam from this second-level system, such as 1 in every 100 messages belonging to a new spam class getting through, but given that rate, would
like as few legitimate messages as possible to be rejected.

Let $\delta_{\cQ}$ be the maximally allowed type II error.
Then the dual SCC problem is:
\begin{align*}
\argmin_r \;&\rho(r,P)\\
\textrm{such that:}  &\;\;\rho(r,Q) \geq 1 - \delta_{\cQ},\;\; \forall Q \in \cQ,
\end{align*}
where $r(\cdot)$ is any function $\Omega \to [0,1]$.
When $\Omega$ is discrete and finite, this problem
has a finite number of variables and a possibly infinite number of constraints depending on $\cQ$.
Thus, it is a linear semi-infinite program.

We represent by $r^*_I(\cdot)$ a solution to the primal problem, and by $r^*_{II}(\cdot)$ a solution
to the dual problem.
Define $\delta^* \eqdef \rho(r^*_{II}, P)$ and
$\delta^*_{\cQ} \eqdef 1 - \min_{Q \in \cQ} \rho(r^*_I, Q)$.
Since $r(\omega) \equiv \delta$ and $r(\omega) \equiv \delta_{\cQ}$
are respectively feasible solutions to the primal and dual problems, $\delta^* \leq \delta_{\cQ}$ and $\delta \leq \delta^*_{\cQ}$.

\begin{lemma}
\label{lem: dual-tight}
Let $\Omega$ be finite and discrete.
If $\delta^*_{\cQ} > 0$, then $\rho(r^*_I,P) = \delta$.
If $\delta^* > 0$, then $\min_Q \rho(r^*_{II}, Q) = 1 - \delta_{\cQ}$.
\end{lemma}
\begin{proof}
Let $\delta^*_{\cQ} > 0$. Let us assume by contradiction that $\rho(r^*_I, P) < \delta$.
Then, define $r'(\omega) \eqdef \min\left\{1, r^*_I(\omega) + \frac{\delta - \rho(r^*_I,P)}{N}\right\}$.
Clearly, $\rho(r',P) \leq \delta$ and $\min_Q \rho(r',Q) > \min_Q \rho(r^*_I,Q)$. Contradiction.

Let $\delta^* > 0$. Let us assume by contradiction that $\min_Q \rho(r^*_{II}, Q) > 1 - \delta_{\cQ}$.
Then, define $r''(\omega) \eqdef \max\left\{0, r^*_{II}(\omega) - \frac{\min_Q \rho(r^*_{II},Q) - (1 - \delta_{\cQ})}{N}\right\}$.
Clearly, $\min_Q \rho(r'',Q) \geq 1 - \delta_{\cQ}$ and $\rho(r'',P) < \rho(r^*_{II},P)$. Contradiction.
\end{proof}

We define $R^I_\delta \eqdef R^*_\delta$ as the set of primal-optimal rejection
functions, and $R^{II}_{\delta_{\cQ}}$ as the set of dual-optimal rejection functions.
Examining the dual SCC problem in the investment analogy, the learner is
assigned a target amount of money, $1 - \delta_{\cQ}$, which must be obtained
on selling all assets. The learner's goal is to achieve this with the minimal
starting investment. We can see that if the learner invests no money,
then the amount of money made will fall short of the target. By investing in assets with higher ROI,
the learner makes the most amount of progress towards the target with the
least amount of money invested. Thus, we can see that the optimal investment strategy
is likely to be similar to that of the primal problem.
In fact, as shown by the following theorem, under mild conditions, the two
sets of optimal strategies are identical.

\begin{theorem}[Primal-Dual Equivalence]
\label{thm: dual-equiv}
Let $\Omega$ be finite and discrete.
If $\delta > 0$ and $\delta^*_{\cQ} > 0$, then $R^{II}_{\delta^*_{\cQ}} = R^I_\delta$.
If $\delta_{\cQ} > 0$ and $\delta^* > 0$, then $R^{I}_{\delta^*} = R^{II}_{\delta_{\cQ}}$.
\end{theorem}
\begin{proof}
Let $\delta > 0$ and $\delta^*_{\cQ} > 0$. By Lemma~\ref{lem: dual-tight}, $\rho(r^*_I,P) = \delta$.
Clearly, $r^*_I$ is a feasible solution to the dual problem with $\delta_{\cQ} = \delta^*_{\cQ}$.
Thus, $\delta^* \leq \delta$.
Let us assume by contradiction that $\delta^* < \delta$. Then, there must exist some $r^*(\cdot)$ such that $\min_Q \rho(r^*,Q) \geq 1 - \delta^*_{\cQ}$
and $\rho(r^*,P) < \delta$.
Define $r'(\omega) \eqdef \min\left\{1, r^*(\omega) + \frac{\delta - \rho(r^*,P)}{N}\right\}$.
Then, clearly $\rho(r',P) \leq \delta$, but $\min_Q \rho(r',Q) > \min_Q \rho(r^*,Q) \geq 1 - \delta^*_{\cQ} = \min_Q \rho(r^*_I, Q)$. Contradiction. Therefore, $\delta^* = \delta$.

Since $\delta^* = \delta > 0$, by Lemma~\ref{lem: dual-tight}, $\min_Q \rho(r^*_{II}, Q) = 1 - \delta_{\cQ}$.
Thus, $r \in R^I_\delta$ if $\min_Q \rho(r,Q) = 1 - \delta^*_{\cQ}$ and $\rho(r,P) = \delta$.
Likewise, $r \in R^{II}_{\delta^*_{\cQ}}$ if $\min_Q \rho(r,Q) = 1 - \delta^*_{\cQ}$ and $\rho(r,P) = \delta$.
Therefore, $R^{II}_{\delta^*_{\cQ}} = R^I_\delta$.

Let $\delta_{\cQ} > 0$ and $\delta^* > 0$. By Lemma~\ref{lem: dual-tight}, $\min_Q \rho(r^*_{II}, Q) = 1 - \delta_{\cQ}$.
Clearly, $r^*_{II}$ is a feasible solution to the primal problem with $\delta = \delta^*$.
Thus, $\delta^*_{\cQ} \leq \delta_{\cQ}$.
Let us assume by contradiction that $\delta^*_{\cQ} < \delta_{\cQ}$. Then, there must exist some $r^*(\cdot)$ such that $\min_Q \rho(r^*,Q) > 1 - \delta_{\cQ}$
and $\rho(r^*,P) \leq \delta^*$.
Define $r''(\omega) \eqdef \max\left\{0, r^*(\omega) - \frac{\min_Q \rho(r^*,Q) - (1 - \delta_{\cQ})}{N}\right\}$.
Then, clearly $\min_Q \rho(r'',Q) \geq 1 - \delta_{\cQ}$, but $\rho(r'',P) < \rho(r^*,P) \leq \delta^* = \rho(r^*_{II},P)$. Contradiction. Therefore, $\delta^*_{\cQ} = \delta_{\cQ}$.

Since $\delta^*_{\cQ} = \delta_{\cQ} > 0$, by Lemma~\ref{lem: dual-tight}, $\rho(r^*_I,P) = \delta = \delta^*$.
Thus, $r \in R^I_\delta$ if $\min_Q \rho(r,Q) = 1 - \delta_{\cQ}$ and $\rho(r,P) = \delta^*$.
Likewise, $r \in R^{II}_{\delta_{\cQ}}$ if $\min_Q \rho(r,Q) = 1 - \delta_{\cQ}$ and $\rho(r,P) = \delta^*$.
Therefore, $R^{I}_{\delta^*} = R^{II}_{\delta_{\cQ}}$.
\end{proof}

Using Theorem~\ref{thm: dual-equiv}, it is trivial to solve the dual SCC problem where $\cQ = \cQ_\Lambda = \{Q: D_P(Q) \geq \Lambda\}$.
To begin with, since we assume that $p_i > 0$ for all $i$, note that $\delta_{\cQ} < 1 \Rightarrow \delta^* > 0$. Therefore, since $\delta_{\cQ} > 0$
the optimal solution sets are identical by Theorem~\ref{thm: dual-equiv}, and all the intermediate results, including Theorem~\ref{thm:gd-2-points}, are correct when
solving the primal problem with $\delta = \delta^* > 0$. Therefore, it is trivial to construct a dual-analogue to Theorem~\ref{thm:linear_program}. We also prove
the analogue to Lemma~\ref{lem:lin-prog-delta}.
\begin{theorem}[Dual SCC Linear Program]
\label{thm:dual-linear_program}
An optimal soft rejection function
and the optimal type I error, $z_I$, is obtained by solving the following
linear program:
\begin{align}
\label{eqn:dual-linprog}
  \text{minimize}_{r_1, r_2, \dots, r_K, z_I} \;\;z_I, \;&\; \text{subject to:}\nonumber\\
                    & \sum_{i=1}^{K}{r_i|S_i|p(S_i)} \leq z_I\\
                 &1 \geq r_1 \geq r_2 \geq \dots \geq r_K \geq 0\nonumber\\
                 &r_w \geq 1 - \delta_{\cQ}\nonumber\\
                 &\rho_i \geq 1 - \delta_{\cQ}, \; i \in \{1, 2, \dots, |\cL||\cH|+|\cM|\} .\nonumber
\end{align}
\end{theorem}
\begin{lemma}
\label{lem:dual-lin-prog-delta}
Let $r^*$ be the solution to the linear program. If $r^*$ is vulnerable, then $r^* = r^{1-\delta_{\cQ}}$.
\end{lemma}
\begin{proof}
Let $r^*$ be a vulnerable solution to the linear program~(\ref{eqn:dual-linprog}),
which clearly satisfies
$1 \geq r^*_1 \geq r^*_2 \geq \dots \geq r^*_K \geq 0$.
Therefore, for all $i \in I_{min}(r^*)$, $r^*(i) = r^*_K$.
We define $z^*_I$ to be the minimal value of $z_I$ that the linear program achieves for $r^*$.
Let $j = \argmin_{i \in I_{min}(r^*)} p_i$ and let $S_u$ be the level set to
which $j$ belongs.
We now prove that $u=1$.

We first deal with the case where $D^{S_K}_P \geq \Lambda$ (in which case the constraint is completely vacuous).
We note in this case that $S_1, S_2, \dots, S_K \in \cL \bigcup \cM$, and therefore $w = K$.
Thus, we have $r^*_1 \geq r^*_2 \geq \dots \geq r^*_K = r^*_w \geq 1-\delta_{\cQ}$.
Therefore, $\sumk |S_i|p(S_i) r^*_i \geq r^*_K \geq 1-\delta_{\cQ}$. Therefore, $z^*_I \geq 1-\delta_{\cQ}$.
We note that $r_1 = r_2 = \dots r_K = 1-\delta_{\cQ}$ is a valid solution to the linear program
for which $z_I = 1-\delta_{\cQ}$, which is the minimal value achievable. Therefore, $z^*_I = 1-\delta_{\cQ}$.
If $u>1$, then $r^*_1 > r^*_K \geq z^*_I = 1-\delta_{\cQ}$ and $\sumk |S_i|p(S_i) r^*_i > r^*_K \geq 1-\delta_{\cQ}$.
Therefore, if $D^{S_K}_P \geq \Lambda$, $u = 1$.

We now turn our attention to the case where $D^{S_K}_P < \Lambda$. If we assume by contradiction that $u > 1$,
then $r^*_{u-1} > r^*_u = r^*_{u+1} = \dots = r^*_K$.
$D_P\left(X^{(j)}\right) \geq \Lambda$ and by our assumption, $D^{S_K}_P < \Lambda$, which implies that $S_K \in \cH$.
If $D_P\left(X^{(j)}\right) > \Lambda$,
then there exists some $l$ for which $j \in S_u \equiv S_l \in \cL$.
The rejection rate for the level-set pair, $(l,K)$,
is $r^*_K \geq 1-\delta_{\cQ}$.
Otherwise, $D_P\left(X^{(j)}\right) = \Lambda$, and $j \in S_u \equiv S_m \in \cM$, and
we have a rejection rate for $S_m$ of $r^*_m = r^*_K$ and since $r_w \geq 1-\delta_{\cQ}$,
this implies that $r^*_K = r^*_m = r_w \geq 1-\delta_{\cQ}$.
Therefore, since $u > 1$, we get that $\sumk |S_i|p(S_i) r^*_i > r^*_K \geq 1-\delta_{\cQ}$.
Therefore, if $D^{S_K}_P < \Lambda$, $u = 1$.

Therefore, $u=1$. This results in $r^*_1 = r^*_2 = \dots = r^*_K \geq 1-\delta_{\cQ}$.
Therefore, $\sumk |S_i|p(S_i) r^*_i = r^*_K$ is clearly minimized by $r^*_K = 1-\delta_{\cQ}$,
or $r^* = r^{1-\delta_{\cQ}}$.
\end{proof}


%
%
%
%

\section{Concluding Remarks}

We have introduced a game-theoretic approach to the SCC problem.
In this approach the learner is opposed by an adversary. We believe that this viewpoint
is essential for analyzing SCC applications such as intrusion detection and, in general,
for ``agnostic''
analysis of single-class classification.
This game-theoretic view lends itself
well to analysis, allowing us to prove under what conditions low-density rejection is hard-optimal and
if an optimal monotone rejection function is guaranteed to exist.
Our analysis introduces soft decision strategies, which potentially allow for
significantly better
performance in our adversarial setting.

Observing the learner's futility when facing an omniscient and unlimited adversary, we considered
restricted adversaries and provided full analysis of an interesting family of constrained games (in a decision-theoretic ``Bayesian''
setting where $P$ is assumed to be known).
The constraint we imposed on the adversary, given in terms of a divergence gap
between the target and opposing distributions, is inspired by similar constraints
used in ``two-sample
problem'' related work in information theory \cite{Ziv88,Gut89,ZM93}.
Of course, to compute the optimal learner strategy one has to know the exact value of this
divergence gap, which is unknown in pure SCC problems.
In applications we expect that something will be known or could be hypothesized about
possible opposing distributions. For example, in biometric authentication,
one should be able to statistically measure this gap.
Thus, one could perhaps determine with a high confidence level that at least 99.9\% of the
population has a distribution with a KL-Divergence of at least 10 from the distribution of any member of the population.
This can obviously be extended to $k$-factor authentication \cite<see, e.g.,>{PointchevalZ08}. Assuming
that the adversary may know $k-1$ factors, gaps can be found for each of the factors in order to ensure
a particular intruder pass rate, with high probability.
A different type of example occurs in extremely unbalanced two-class classification problems.
Here, one could utilize the very few given examples from the other class to infer a bound on the gap.
This complements the results of \citeA{KowalczykR03} where one-class learners were found to out-perform
their two-class counterparts in some settings.

Our final major contribution is a simple and computationally feasible one-class classification algorithm.
The SCC classifier is generated by thresholding a soft two-class classifier's output, where the output
serves as a proxy for a density estimate, and a quantile estimate serves as the threshold. This approach can be extended to other use cases.
For example, in \cite{YehLeeLee09}, a multi-class classification problem is solved
by constructing SVDD \cite{TaxD99} one-class classifiers, where each class is described by a sphere, and learning a discriminant function which
assigns a test point the class whose sphere center-point it is ``nearest'' to
(the distance is normalized by various statistics). Instead of using SVDD,
our approach would be to create a
two-class classifier for each class, where the second class is uniformly distributed over
the active cells. A test point would be passed to each two-class classifier and the class chosen would be that
belonging to the classifier which ranked the test point in the highest quantile (relative to the training sample for each class). This classification scheme makes sense because it labels the test point with
the class for which it has the highest ``relative'' density (relative to other points within each class).
Thus, we achieve the same goal without resorting to heuristics.

We have introduced a dual SCC problem and shown that, under very weak conditions, the solution
sets for the primal and dual problems coincide. This allows one to easily extend results from one setting
to the other, as we demonstrated by providing the dual solution to the constrained family of games considered earlier.

Various extension and generalizations to these results can be found in \cite{Nisenson2010}.
These include extensions of Section~\ref{sec:characterization} results
to the infinite discrete setting and
extensions to Section~\ref{sec:cont} giving
additional results in the continuous setting such as a two-class reduction of SCC to
hard binary-classification (as opposed to soft classification as we present here).

Our work can be extended in various ways and
we believe that it opens up new avenues for future research and in particular
could be useful for inspiring new algorithms for finite-sample SCC problems.
One of the most important questions would be
to determine convergence rates for the algorithm given in Section~\ref{subsec:occ_via_2class}.
It would be very nice to obtain an explicit expression for the lower bound
output by the linear program of Theorem~\ref{thm:linear_program}.
Extensions of the analysis and
algorithms for additional feature spaces, such as graphs or time-series, would be useful. An interesting
question is whether performance, whether in terms of type II error or convergence rates, could be improved in different spaces.
Clearly, the utilization of
randomized strategies should be carried over to the finite sample case as well.
A natural desirable extension is to extend
our analysis for the soft setting to continuously infinite spaces.

We have focused in this work on ``single-shot'' games, meaning that the learner
has to make a decision after every test observation. This is a very difficult setting as one cannot utilize
cumulative statistics of the other class. Thus, we would expect that the results could be improved upon in a
repeated-game setting, where several observations are provided from the same distribution,
or in change point detection \cite{Page54,Hinkley70}, where one has to determine in a series of observations where the distribution $P$
has been replaced by the (unknown) distribution $Q$ as the underlying source. In the
finite discrete setting, one should be able to easily extend some of the results here
by replacing events with types \cite{CoverT91}.
Finally, a very interesting
setting to consider is one in which the adversary has partial knowledge of the problem parameters
and the learner's strategy. For example, the adversary may only know that $P$ is in some
subspace.

\bibliography{one_class_arxiv}
\bibliographystyle{apacite}

\newpage
\appendix

\section{Section 5 Proofs}
\label{sec: appendix-monotone-proofs}
As a reminder, we assume that $p_i > 0$ for all $i \in \Omega$.
Furthermore, for convenience we assume w.l.o.g. that $\Omega$ is defined such that $0 < p_1 \leq p_2
\leq \dots \leq p_N$.

\begin{lemma}
\label{lem:propA-rearrange}
Let $a + b = c + d$ and $a + c \geq b + d$. Then, $a \geq d$.
\end{lemma}
\begin{proof}
Clearly $\frac{a}{2} \geq \frac{b + d - c}{2}$. Therefore, $a \geq \frac{a + b + d - c}{2} = d$
\end{proof}

\theoremx{4}{
[Optimal Monotone Hard Decisions]
When the learner is restricted to hard-decisions
and $\cQ$ satisfies $\PropA$ w.r.t.~$P$, then there exists a monotone $r \in \cRdeltastar$.
}
\begin{proof}
Recalling that $0 < p_1 \leq p_2 \leq \cdots \leq p_N$,
we now define a rejection function as being $x$-monotone, if it is monotone up to index $x$.
In other words, a rejection function, $r(\cdot)$ is $x$-monotone if
$p_j < p_k \Rightarrow r(j) \geq r(k)$, for all $j < k \leq x$. Clearly, all rejection
functions are 1-monotone, and a monotone rejection function is $N$-monotone.

Let us assume, by contradiction, that no monotone rejection function exists in $\cRdeltastar$.
We will prove the existence of an $N$-monotone rejection function in $\cRdeltastar$ via induction.
Let $r \in \cRdeltastar$. Then, $r$ is $(k-1)$-monotone but not $k$-monotone, for some $2 \leq k \leq N$.
Let $j = \min \{i : r(i) = 0\}$. We note that $1 \leq j < k$ and $r(k) = 1$ (otherwise, $r$ would be $k$-monotone).
We now prove the existence of a $k$-monotone rejection function, $r^* \in \cRdeltastar$.
We define $r^*$ as follows:
\begin{align*}
r^*(i) = \begin{cases}
    1& i=j,\\
    0& i=k,\\
    r(i)& \text{otherwise}.
    \end{cases}
\end{align*}
Note that for all $i \leq j$, that $r^*(i) = 1$, and for all $j < i \leq k$, that $r^*(i) = 0$.
Thus, $r^*$ is a $k$-monotone rejection function. We now prove that $r^* \in \cRdeltastar$.
Note that $\rho(r^*,P) = \rho(r,P) + p_j - p_k < \rho(r,P) \leq \delta$, and thus $r^*$ is a $\delta$-valid
hard rejection function.
Let $Q^* \in \cQ$ be such that $\min_{Q} \rho(r^*,Q) = \rho(r^*,Q^*) = \rho(r,Q^*) + q^*_j - q^*_k$.
Thus, if $q^*_j \geq q^*_k$, $\rho(r^*,Q^*) \geq \rho(r,Q^*)$. Otherwise, there exists ${Q^*}'$ as in
$\PropA$ and in particular, by Lemma~\ref{lem:propA-rearrange},
$q^*_j \geq {q^*}'_k$. Consequently, $\rho(r^*,Q^*) = \rho(r,{Q^*}') + q^*_j - {q^*}'_k \geq \rho(r,{Q^*}')$.
Therefore, there always exists $Q \in \cQ$ such that $\rho(r^*,Q^*) \geq \rho(r,Q)$
(either $Q = Q^*$ or $Q = {Q^*}'$).
Therefore, $\min_Q \rho(r^*,Q) \geq \min_Q \rho(r,Q)$, and thus, $r^* \in \cRdeltastar$.
Therefore, by induction, there must exist an optimal $N$-monotone rejection
function. Contradiction.
\end{proof}
\begin{remark}
The above proof works for a weaker version of $\PropA$:
If for all $p_j < p_k$ and $Q \in \cQ$ for which $q_j < q_k$, there exists a distribution $Q' \in \cQ$
such that $q'_j - q'_k + \sum_{i=1}^j(q_i - q'_i) + \sum_{i=k+1}\min\{q_i - q'_i,0\} \geq 0$.
As used in the proof, this would read:
\begin{align*}
\rho(r^*,Q^*) =& \rho(r^*,{Q^*}') + \sumn r^*(i)(q^*_i - {q^*}'_i)\\
    =& \rho(r,{Q^*}') + {q^*}'_j - {q^*}'_k + \sum_{i=1}^j(q^*_i - {q^*}'_i) + \sum_{i=k+1}r^*(i)(q^*_i - {q^*}'_i)\\
    \geq& \rho(r,{Q^*}') + {q^*}'_j - {q^*}'_k + \sum_{i=1}^j(q^*_i - {q^*}'_i) + \sum_{i=k+1}\min\{q^*_i - {q^*}'_i,0\}\\
    \geq& \rho(r,{Q^*}').
\end{align*}
\end{remark}
\begin{remark}
If we strengthen the condition in $\PropA$ from
$q_j + q'_j \geq q_k + q'_k$ to $q_j + q'_j> q_k + q'_k$
for all distributions $Q$ such that $q_j \leq q_k$ (instead of $q_j < q_k$),
then all optimal rejection functions would be monotone.
Note that the set of all distributions
does not have this modified property, but the set of all distributions bounded away from zero
($\{Q: q_i > 0, \forall i \in \Omega\}$) does.
\end{remark}

\theoremx{6}{[Optimal Monotone Soft Decisions]~\\
If $\cQ$ satisfies $\PropB$ w.r.t.~$P$, then there exists an optimal strictly monotone rejection function.
}
\begin{proof}
We note that the condition for strict-monotonicity is equivalent to $p_j \leq p_k \Rightarrow r(j) \geq r(k)$,
and that $0 < p_1 \leq p_2 \leq \cdots \leq p_N$.
We now define an $x$-right-strictly-monotone rejection function as one
which has strictly-monotone properties for the last $x$ indices. In other words,
a rejection function $r(\cdot)$ is $x$-right-strictly-monotone if
$p_j \leq p_k \Rightarrow r(j) \geq r(k)$, for all $j < k$, $k > N - x$. Clearly, all
rejection functions are 0-right-strictly-monotone, and an $N$-right-strictly monotone rejection
function is strictly monotone.

We assume contradictorily that there is no such rejection function. Let $r \in \cRdeltastar$.
We note that $r$ is $(v-1)$-right-strictly-monotone but not $v$-right-strictly-monotone
for some $1 \leq v \leq N$. We will prove by induction that there exists an $N$-right-strictly-monotone
function in $\cRdeltastar$. Let $k = N - v + 1$. Since $r$ is not $v$-right-strictly-monotone, then
there must exist some $j < k$ for which $p_j \leq p_k$ and $r(j) < r(k)$.
Define, for any event $\omega$ and distribution $D$:
\begin{align*}
S_r(\omega) \eqdef& \{i: p_i = p_\omega \wedge r(i) = r(\omega)\};\\
g(D, \omega) \eqdef& \frac{\sum_{i \in S_r(\omega)} d_i}{|S_r(\omega)|p_\omega} = \frac{\sum_{i \in S_r(\omega)} \frac{d_i}{p_\omega}}{|S_r(\omega)|}.
\end{align*}
$S_r(\omega)$ is the intersection of $\omega$'s probability level-set with $\omega$'s rejection level-set.
$g(D, \omega)$ is simply an average of the elements of $D$ corresponding to symbols in $S_r(\omega)$ normalized
by $\frac{1}{p_\omega}$.
We note that $g(P, \omega) = 1$ always. We define $r^*$ as follows:
\begin{align*}
r^*(i) =& \begin{cases}
    \frac{|S_r(j)|p_jr(j) + |S_r(k)|p_kr(k)}{|S_r(j)|p_j + |S_r(k)|p_k}& i \in S_r(j) \cup S_r(k),\\
    r(i)& \text{otherwise};
    \end{cases}\\
    \\
\Rightarrow r^*(j) - r(j) =& \frac{|S_r(j)|p_jr(j) + |S_r(k)|p_kr(k)}{|S_r(j)|p_j + |S_r(k)|p_k} - r(j)\\
                             =& \frac{|S_r(k)|p_k(r(k) - r(j))}{|S_r(j)|p_j + |S_r(k)|p_k} > 0\\
r(k) - r^*(k) =& \frac{|S_r(j)|p_j(r(k) - r(j))}{|S_r(j)|p_j + |S_r(k)|p_k} = \frac{|S_r(j)|p_j}{|S_r(k)|p_k}(r^*(j) - r(j))\\
\Rightarrow \forall D, \rho(r^*,D) - \rho(r,D) =& \left[(r^*(j) - r(j))\sum_{i \in S_r(j)}d_i\right] + \left[(r^*(k) - r(k))\sum_{i \in S_r(k)}d_i\right]\\
                                           =& (r^*(j) - r(j))\left[\sum_{i \in S_r(j)}d_i
                                                         - \frac{|S_r(j)|p_j}{|S_r(k)|p_k}\sum_{i \in S_r(k)}d_i\right]\\
                                           =& (r^*(j) - r(j))|S_r(j)|p_j\left[\frac{\sum_{i \in S_r(j)}d_i}{|S_r(j)|p_j}
                                                                        - \frac{\sum_{i \in S_r(k)}d_i}{|S_r(k)|p_k}\right]\\
                                           =& (r^*(j) - r(j))|S_r(j)|p_j\left[g(D,j) - g(D,k)\right] .
\end{align*}
Therefore, noting that $r^*(j) > r(j)$,
\begin{equation}
\label{eqn:first_conclusion}
\rho(r^*,D) < \rho(r,D) \Rightarrow \;
    g(D,j) < g(D,k) \Rightarrow \; \min_{i \in S_r(j)} \frac{d_i}{p_j} < \max_{i \in S_r(k)} \frac{d_i}{p_k} .
\end{equation}
Since $g(P,j) = g(P,k) = 1$, $\rho(r^*,P) = \rho(r,P) = \delta$. Therefore, $r^*$ is a valid rejection
function.
Let $u > k$. We note by the definition of $r^*$ and the fact that $r$ is $(v-1)$-right-strictly-monotone
that $r^*(u) = r(u) \leq r(j) < r^*(j) = r^*(k) < r(k)$. Therefore, $r^*$ is still $(v-1)$-right-strictly-monotone
(but not necessarily $v$-right-strictly-monotone).

Let $Q^*$ be such that $\rho(r^*,Q^*) = \min_Q \rho(r^*,Q)$. We will now show that $\exists \hat{Q} \in \cQ$ s.t.
$\rho(r^*,Q^*) \geq \rho(r,\hat{Q})$ (and therefore, $\min_Q \rho(r^*,Q) \geq \min_Q \rho(r,Q)$). The following algorithm
finds such a $\hat{Q}$:
\begin{enumerate}
    \item Set $Q = Q^*$.
    \item while $\rho(r^*,Q) < \rho(r,Q)$
    \begin{enumerate}
        \item Let $a$ and $b$ be such that $q_a = \min_{i \in S_r(j)} q_i$ and $q_b = \max_{i \in S_r(k)} q_i$.
                We note that $\rho(r^*,Q) < \rho(r,Q) \Rightarrow \frac{q_a}{p_j} < \frac{q_b}{p_k} \Rightarrow
                \frac{q_a}{p_a} < \frac{q_b}{p_b}$.
        \item Since $\cQ$ satisfies $\PropB$, there exists a $Q' \in \cQ$ which is identical to $Q$ for all $i \neq a,b$
                and such that $\frac{q'_a}{p_a} \geq \frac{q'_b}{p_b}$. Set $Q = Q'$.
    \end{enumerate}
    \item end while. Output $\hat{Q} = Q$.
\end{enumerate}
Since for all iterations, $r^*(a) = r^*(b)$, at step (b) we have $\rho(r^*,Q') = \rho(r^*,Q) = \rho(r^*,Q^*)$.
After setting $Q = Q'$ at step (b), we have
$\frac{q_a}{p_a} \geq \frac{q_b}{p_b}$,
and therefore the loop never repeats for the same pair of symbols $(a,b)$.
Therefore, the loop is guaranteed to terminate. After ending, $\rho(r^*,Q^*) = \rho(r^*,\hat{Q}) \geq \rho(r,\hat{Q})$,
so $\min_Q \rho(r^*,Q) \geq \min_Q \rho(r,Q)$, and $r^* \in \cRdeltastar$.

While there still exists a $j$ such that $r^*(j) < r^*(k)$ we
relabel $r^*$ as $r$ and repeat the above procedure (note that
it never repeats for the same pair $j,k$). The resulting $r^*$ is $(v-1)$-right-strictly-monotone
as shown above, but since now $j < k \Rightarrow r^*(j) \geq r^*(k)$, $r^*$ is $v$-right-strictly-monotone.

Thus, by induction there exists an optimal $N$-right-strictly-monotone rejection function, which is a contradiction.
\end{proof}

\begin{remark}
\label{rm: chr-monotonic-or-equality}

Strengthening the conditions in $\PropB$ to $\frac{q_j}{p_j} \leq \frac{q_k}{p_k}$ and
$\frac{q'_j}{p_j} > \frac{q'_k}{p_k}$ would strengthen Theorem~\ref{thm:chr-monotonic-equality}
so that all optimal rejection functions are strictly monotone. Once more, the set of all distributions
does not have this modified property, but the set of all distributions bounded away from zero does.
\end{remark}

\theoremx{10}{[LDRS optimality]
Let $r^*$ be an LDRF. Let $r$ be any monotone $\delta$-valid
rejection function. Then
\begin{equation*}
\min_{Q \in \cQ} \rho(r^*,Q) \geq \min_{Q \in \cQ} \rho(r,Q),
\end{equation*}
for any $\cQ$
satisfying $\PropC$. Thus, if $\cQ$ possess both $\PropA$ and $\PropC$ w.r.t.~$P$, then LDRS is
hard-optimal.
}
\begin{proof}
We define, for a hard rejection function $r$, $\theta(r) \eqdef \min_{\omega : r(\omega) = 0} p_\omega$,
$Z_\theta(r) \eqdef \{\omega : p_\omega = \theta(r) \wedge r(\omega) = 1\}$ and $z_\theta(r) \eqdef |Z_\theta(r)|$.

Assume, by contradiction, that
$\min_{Q \in \cQ} \rho(r^*,Q) < \min_{Q \in \cQ} \rho(r,Q)$. Let $Q^*$ be the minimizer of $\rho(r^*,Q)$.
Then, $\rho(r^*,Q^*) < \rho(r,Q^*)$.
If $\theta(r) > \theta(r^*)$ then, by the definition of LDRF and by the monotonicity of $r$, $\rho(r,P) > \delta$,
which contradicts $r$'s validity. If $\theta(r) < \theta(r^*)$ then, by $r$'s monotonicity,
$r(\omega) = 1 \Rightarrow r^*(\omega) = 1$, and for any distribution
$D$, $\rho(r,D) \leq \rho(r^*,D)$,
contradicting $\rho(r^*,Q^*) < \rho(r,Q^*)$.
Therefore, $\theta(r) = \theta(r^*)$.
If $z_\theta(r) > z_\theta(r^*)$ then $\rho(r,P) > \delta$ since $r^*$ is an LDRF.
Otherwise, $z_\theta(r) \leq z_\theta(r^*)$, and by $\PropC$ the set $\cQ$ contains all distributions identical to $Q^*$ up to a permutation of the
$\theta$-probability events. Therefore, $\min_{Q \in \cQ} \rho(r^*,Q) \geq \min_{Q \in \cQ} \rho(r,Q)$.
Contradiction.
\end{proof}

\section{Section 6 Proofs}
\label{sec: appendix-games-proof}

\lemmax{24}{Let $r^*$ be the solution to the linear program. If $r^*$ is vulnerable, then $r^* = r^\delta$.
}
\begin{proof}
Let $r^*$ be a vulnerable solution to the linear program~(\ref{eqn:linprog}),
which clearly satisfies
$1 \geq r^*_1 \geq r^*_2 \geq \dots \geq r^*_K \geq 0$.
Therefore, for all $i \in I_{min}(r^*)$, $r^*(i) = r^*_K$.
We define $z^*$ to be the maximal value of $z$ that the linear program achieves for $r^*$.
Let $j = \argmin_{i \in I_{min}(r^*)} p_i$ and let $S_u$ be the level set to
which $j$ belongs.
We now prove that $u=1$.

We first deal with the case where $D^{S_K}_P \geq \Lambda$ (in which case the constraint is completely vacuous).
We note in this case that $S_1, S_2, \dots, S_K \in \cL \bigcup \cM$, and therefore $w = K$.
Thus, we have $r^*_1 \geq r^*_2 \geq \dots \geq r^*_K = r^*_w \geq z^*$.
We note that $z^* \leq \delta$, otherwise $\delta = \sumk |S_i|p(S_i) r^*_i \geq r^*_K \geq z^* > \delta$.
We note that $r_1 = r_2 = \dots r_K = \delta$ is a valid solution to the linear program
for which $z = \delta$, which is the maximal value achievable. Therefore, $z^* = \delta$.
If $u>1$, then $r^*_1 > r^*_K \geq z^* = \delta$ and $\sumk |S_i|p(S_i) r^*_i > r^*_K \geq \delta$.
Therefore, if $D^{S_K}_P \geq \Lambda$, $u = 1$.

We now turn our attention to the case where $D^{S_K}_P < \Lambda$. If we assume by contradiction that $u > 1$,
then $r^*_{u-1} > r^*_u = r^*_{u+1} = \dots = r^*_K$. We define $\Pr[S] \eqdef |S|p(S)$ for a level set $S$, and
$c \eqdef \frac{\Pr[S_{u-1}]}{\sum_{i=u}^K \Pr[S_i]}$.
Let $0 < \epsilon < \frac{r^*_{u-1} - r^*_u}{c + 1}$. We now define a new rejection function $r'$ as follows:
\begin{align*}
r'_i \eqdef& \begin{cases}
    r^*_i & i < u - 1,\\
    r^*_i - \epsilon& i = u-1,\\
    r^*_i + c\epsilon& i \geq u.
    \end{cases}
\end{align*}
We note that:
\begin{align*}
\rho(r',P) =& \rho(r^*,P) - \Pr[S_{u-1}]\epsilon + \sum_{i = u}^{K} \Pr[S_i]c\epsilon \\
        =& \rho(r^*,P) - \Pr[S_{u-1}]\epsilon + \Pr[S_{u-1}]\epsilon = \rho(r^*,P) = \delta.
\end{align*}
Therefore, $r'$ is $\delta$-valid.
Let $z'$ be the maximal value of $z$ that the linear program
achieves for $r'$.
$D_P\left(X^{(j)}\right) \geq \Lambda$ and by our assumption, $D^{S_K}_P < \Lambda$, which implies that $S_K \in \cH$.
If $D_P\left(X^{(j)}\right) > \Lambda$,
then there exists some $l$ for which $j \in S_u \equiv S_l \in \cL$.
The rejection rate for the level-set pair, $(l,K)$,
is $r^*_K$.
Otherwise, $D_P\left(X^{(j)}\right) = \Lambda$, and $j \in S_u \equiv S_m \in \cM$, and
we have a rejection rate for $S_m$ of $r^*_m = r^*_K$.
Since $z^*$ cannot be less than $r^*_{min} = r^*_K$, we have $z^* = r^*_K$ in both cases ($r^*_K \geq z^* \geq r^*_K$).
We note that:
$$
r'_{u-1} - r'_u = (r^*_{u-1} - \epsilon) - (r^*_u + c\epsilon) = (r^*_{u-1} - r^*_u) - (c + 1)\epsilon > 0
$$
Clearly for  $i > u$, $r'_{i-1} = r'_i = r^*_K + c\epsilon$.
Obviously, $1 \geq r'_1$ and $r'_K > r^*_K \geq 0$.
Therefore, $1 \geq r'_1 \geq r'_2 \geq \dots \geq r'_K > 0$, and $r'$ is a feasible solution
to the linear program. Furthermore, $z' \geq r'_{min} = r'_K > r^*_K = z^*$, which contradicts the fact
that $r^*$ maximizes $z$ (and is the solution to the linear program).

Therefore, $u=1$. This results in $r^*_1 = r^*_2 = \dots = r^*_K$.
Since $\rho(r^*,P) = \delta$, we have that $\delta = \sumk |S_i|p(S_i) r^*_i = r^*_K$, or $r^* = r^\delta$.
\end{proof}

\section{Section 7 Proofs}
\label{sec: appendix-cont-ldrs-proof}
We begin by providing some additional definitions.
Let $\bB$ be the set of all Borel sets over $\Real^d$.
For two Borel sets $a, b$ we define  $a \lebeq b \Leftrightarrow \lambda(a \symdif b) = 0$,
where $\symdif$ is the symmetric difference operator.
For two functions, $f,g$ over $\Real^d$ and Borel set $b$, define
$\Delta_b(f,g) \eqdef \{x \in b: f(x) \neq g(x)\}$. We define the function $\ind_b(x) \eqdef \ind(x \in b)$, where $\ind(\cdot)$ is the indicator function.
\begin{lemma}
Let $m' \in \cldp$. Let $m$ be a Borel set such that $m \lebeq m'$. Then $m \in \cldp$.
\end{lemma}
\begin{proof}
Since $m' \in \cldp$, $P(m') = \delta$ and there exists a minimum volume set $b'$ of measure $1-\delta$,
such that $m' \bigcap b' = \emptyset$. Let $b \eqdef b' \setminus m$.
We note that $\lambda(m \symdif m') = 0$.
Therefore, $b = b' \setminus m \lebeq b' \setminus m' = b'$.
Therefore, $b'$ is a minimum volume set of measure $1-\delta$. Since $m \bigcap b = \emptyset$ and
$P(m) = \delta$, $m \in \cldp$.
\end{proof}

\theoremx{30}{[LDRS optimality - Continuous Setting]
When the learner is restricted to hard-decisions and $\cQ$ satisfies $\PropAcont$ w.r.t.~$P$,
then LDRS is optimal.
}
\begin{proof}
Assume that the statement is false.
Therefore, there must exist some $m \in \cldp$ such that for all $r \in R^*_\delta$, $\lambda(\Delta_{\Real^d}(r, \ind_{l_P(m)})) > 0$.
Let $r' \in R^*_\delta$, such that $\rho(r',P) = \delta$. Define
\begin{align*}
r(x) \eqdef \begin{cases}
    1& p(x) = 0,\\
    r'(x)& \text{otherwise}.
    \end{cases}
\end{align*}
Therefore, $r \in R^*_\delta$ and $\lambda(\Delta_{\Real^d}(r, \ind_{l_P(m)})) > 0$.
Let $j = \{x \in m: r(x) = 0\}$. Therefore, $P(j) > 0$. Thus, there must exist
a set $k$, such that $k \bigcap m = \emptyset$, $k \subset \supp(p)$, $\int_k r(x)\lambda(dx) = \lambda(k)$,
and $P(k) = P(j)$ (otherwise, $\rho(r',P) \neq \delta$).
Since $P(k) = P(j) > 0$ and $m \in \cldp$, we have $\lambda(j) \geq \lambda(k)$.
We define:
\begin{align*}
r^*(x) \eqdef \begin{cases}
    1& x \in j,\\
    0& x \in k,\\
    r(x)& \text{otherwise}.
    \end{cases}
\end{align*}
We note that $\rho(r^*, P) = \rho(r, P) \leq \delta$.

Let $Q^* \in \cQ$ be such that $\min_Q \rho(r^*, Q) = \rho(r^*, Q^*) = \rho(r, Q^*) + Q^*(j) - Q^*(k)$.
Thus, if $Q^*(j) \geq Q^*(k)$, $\rho(r^*, Q^*) \geq \rho(r, Q^*)$. Otherwise,
there exists ${Q^*}'$ as in $\PropAcont$ and in particular, by Lemma~\ref{lem:propA-rearrange},
$Q^*(j) > {Q^*}'(k)$. Consequently, $\rho(r^*, Q^*) = \rho(r, {Q^*}') + Q^*(j) - Q^*(k) \geq \rho(r,{Q^*}')$.
Therefore, there always exists $Q \in \cQ$ such that $\rho(r^*,Q^*) \geq \rho(r,Q)$ (either $Q=Q^*$
or $Q={Q^*}'$). Therefore, $\min_Q \rho(r^*,Q) \geq \min_Q \rho(r,Q)$, and thus $r^* \in R^*_\delta$.
However, $\lambda(\Delta_{\Real^d}(r^*, \ind_{l_P(m)})) = 0$.
Contradiction.
\end{proof}

\end{document}